\documentclass{l4dc2027}

\usepackage{times}
\usepackage{booktabs}
\usepackage{array}
\usepackage[figuresright]{rotating}

\newcommand{\CVaR}{\operatorname{CVaR}}

\newcommand{\lam}{\lambda}
\newcommand{\levels}{\mathcal{L}}

\newcommand{\Nevals}{N_{\mathrm{eval}}}
\newcommand{\Rspread}{R}

\newcommand{\dsCEMstyleFull}{dsCEM-style (substituted point set)}
\newcommand{\dsCEMstyle}{dsCEM-style}

\newcommand{\asymci}[3]{\ensuremath{#1\,{}_{\scriptscriptstyle -#2}^{\scriptscriptstyle +#3}}}
\newcommand{\upperci}[2]{\ensuremath{#1\,{}^{\scriptscriptstyle +#2}}}
 
\makeatletter
\renewcommand*{\@jmlrenddoc}{\FloatBarrier\global\let\@reprint\@empty}
\makeatother
\title[QPP: deterministic level-curve risk evaluation]{Quantile Propagation Planning for Budgeted Tail-Risk Evaluation}
\coltauthor{\Name{Weijia Han}\textsuperscript{1},\enspace
  \Name{Zhi Huang}\textsuperscript{2},\enspace
  \Name{Lisha Qu}\textsuperscript{1},\enspace
  \Name{Yuanyuan Tan}\textsuperscript{3},\\
  \Name{Jiayue Meng}\textsuperscript{4},\enspace
  \Name{Wenjun Chen}\textsuperscript{5},\enspace
  \Name{Liying Liang}\textsuperscript{3},\enspace and\enspace
  \Name{Bangjie Deng}\textsuperscript{6}\\
  \addr\footnotesize \textsuperscript{1}U. Washington\enspace
  \textsuperscript{2}U. Rochester\enspace
  \textsuperscript{3}Independent Researcher\enspace
  \textsuperscript{4}Georgia Institute of Technology\enspace
  \textsuperscript{5}Columbia University\enspace
  \textsuperscript{6}U. Penn}
\hypersetup{
  pdftitle={Quantile Propagation Planning for Budgeted Tail-Risk Evaluation},
  pdfauthor={Weijia Han; Zhi Huang; Lisha Qu; Yuanyuan Tan; Jiayue Meng;
    Wenjun Chen; Liying Liang; Bangjie Deng}
}

\begin{document}

\maketitle

\begin{abstract}
Risk-aware planning must identify rare, costly outcomes using only a limited
number of model evaluations. Monte Carlo rollouts spend this budget on noisy
tail estimates, while planning from the mean can miss the tail entirely. We
introduce quantile propagation planning (QPP), which propagates a small, fixed
set of disturbance quantiles through the dynamics and integrates the resulting
cost curve to evaluate both mean and tail risk. The method is exact at the
propagated levels, converges as the quantile grid is refined, and admits a
one-sided conservatism guarantee under stated structural conditions. In controlled
overload and inventory tasks, QPP reduces realized tail cost and constraint
violations relative to mean-rollout planning. On overload, it remains
competitive with sampled evaluators that use substantially more trajectories.
These results position QPP as a transparent, budget-controlled approach to
tail-risk evaluation.
\end{abstract}


\section{Introduction}
\label{sec:intro}

Forecasts guide actions before their outcomes are known. Retail demand
forecasts drive purchasing, inventory, logistics, and pricing
~\citep{fildes2022retail}; building and microgrid controllers schedule energy
from uncertain load and weather~\citep{drgona2020buildingmpc}; and portfolio
optimizers are sensitive to errors in estimated returns and covariances
~\citep{kolm2014portfolio}. Learning-based model predictive control (MPC) faces
the same issue: stochastic learned models must be used without losing
constraint reliability~\citep{hewing2020learningmpc}. Optimization can
therefore amplify a modest predictive error by selecting a different action
and propagating its effect through repeated decisions.

Rare errors also have asymmetric consequences. A full-scale vehicle study
reports that a deterministic MPC baseline using one dynamics model may lose
control in adverse conditions, whereas a risk-averse controller accounts for
uncertain friction and tire parameters~\citep{lew2024racingcvar}. Worst-case
protection, however, can make a plan too conservative to use. The practical
goal is therefore to balance typical performance against costly tail outcomes,
not merely to optimize either the mean or the worst case.

This balance becomes a statistical bottleneck when a planner ranks many
candidates under a fixed model-call budget. At tail mass $\alpha$, empirical
CVaR uses only the worst $\alpha$ fraction of sampled trajectories; with
$\alpha=0.1$ and five rollouts, the estimate is controlled by the single most
costly rollout. Concentration results formalize how the sample budget affects
CVaR accuracy and selection error~\citep{prashanth2020cvarbounds}. Recent CVaR
policy-optimization work likewise traces poor sample efficiency to underused
non-tail trajectories~\citep{mead2025returncapping,luo2026cvarquantile}.
Although those methods study policy learning, they expose the same scarcity of
useful tail data.

Stochastic MPC represents uncertainty through distributions, scenarios, or
probabilistic constraints, but the general nonlinear problem remains
computationally demanding~\citep{mesbah2016smpc}. Risk-aware planners use, for
example, stochastic search over nonlinear beliefs~\citep{wang2021adaptivecvar}
or sample-based CVaR constraints~\citep{lew2024racingcvar}. A recent nonlinear
SMPC framework still reports an accuracy--computation trade-off among
uncertainty-propagation methods~\citep{landgraf2026grampcs}; analytic
propagation avoids sampling variance but typically requires compatible
distributional and value-function structure~\citep{sharma2026momentclosure}.
Mean-only planning is cheaper but can hide skew and harmful tails. The open
question is how to rank many plans by mean and tail cost when each candidate
can afford only a handful of model calls.

We introduce \emph{quantile propagation planning} (QPP), which evaluates $K$
fixed disturbance levels and integrates their piecewise-linear cost curve to
obtain its mean and CVaR. QPP holds the same disturbance rank across time and
coordinates, so it evaluates a comonotonic surrogate rather than the
independently disturbed execution process. This shared-rank coupling can be
conservative and may change which plan is preferred.

Our contributions are:
\begin{itemize}
  \setlength{\itemsep}{1pt}
  \setlength{\parsep}{0pt}
  \setlength{\topsep}{2pt}
  \item A deterministic level-curve evaluator that uses $K$ model trajectories
  and does not require an ordering projection.
  \item Coupled-process guarantees for exactness at the chosen levels,
  grid-refinement convergence, and one-sided conservatism under stated
  conditions.
  \item An evaluation across controlled overload and inventory tasks and a
  fitted microgrid, with budget-matched sampling baselines and ablations that
  identify when QPP is effective.
\end{itemize}
 \section{Related Work}
\label{sec:related}

\paragraph{Sampling-based planning.} PETS~\citep{chua2018pets} propagates
uncertainty through an ensemble. iCEM~\citep{pinneri2020icem} improves CEM
candidates, whereas the \dsCEMstyleFull{} construction~\citep{walker2026dscem}
replaces random candidates by deterministic samples. QPP instead changes
candidate \emph{evaluation}; the composition study therefore crosses QPP
with both samplers and separates candidate-distribution effects from loop
mechanics. On real-load data, the fixed-budget sampling comparison is a
PETS-style bootstrap ensemble (Section~\ref{sec:results-ensemble}).

\paragraph{Risk-aware objectives.} The CVaR representation of
\citet{rockafellaruryasev2000cvar} is what makes the level-curve tail
analytically integrable. Risk-averse MPC supplies the dynamic-control
context~\citep{sopasakis2019riskaverse}. Stochastic-search CVaR control
\citep{wang2021adaptivecvar}, sample-based CVaR constraints
\citep{lew2024racingcvar}, quantile-gradient policy optimization
\citep{luo2026cvarquantile}, distributional RL
\citep{dabney2018qrdqn,dabney2018iqn}, and epistemic value distributions
\citep{luis2023eqr} all retain sampled or learned distributions. QPP's
distinction is deterministic level-curve evaluation inside planning.

\paragraph{Deterministic uncertainty propagation.} Non-Gaussian stochastic
MPC has used deterministic quadrature with stochastic Galerkin
\citep{chen2019stochasticgalerkin}, polynomial-chaos prediction and control
\citep{pan2023stochasticfundamental}, and monotonicity-based reachable sets
\citep{heinlein2025monotonicity}. QPP is narrower: it recursively propagates
one scalar level through an evaluable one-step model and integrates a
piecewise-linear mean/CVaR curve. We claim neither general accuracy nor
computational dominance over those methods.

\paragraph{Learned distributional predictions.} Local propagation of
parametric laws~\citep{petersen2024stabledist}, deep multi-step quantile MPC
\citep{chen2025deepquantilempc}, and multi-horizon quantile forecasting
\citep{wen2017mqrnn} differ from QPP's recursive reuse of one latent level in
a one-step model. Distribution-guided diffusion planning
\citep{rsdiffuser2026}, quantile-regularized representations
\citep{qqworld2026}, and point-estimate CEM~\citep{rezazadeh2025wmms3m} are
related hosts or alternatives, not evidence for nonlinear learned-model
performance here.
 \section{Method}
\label{sec:method}

\paragraph{Setting.} A world model is $z_{t+1} = f(z_t, u_t) + \sigma(z_t,
u_t) \odot \varepsilon_t$ with state $z_t \in \mathbb{R}^d$, action $u_t$,
drift $f$, per-coordinate loading $\sigma$ and a standardized scalar
disturbance $\varepsilon_t$ with quantile function $F^{-1}$; a per-step
cost $c(z_t, u_t)$ is charged on the pre-step state. The planner chooses an
open-loop $u_{0:H-1}$ from $z_0$ and executes it on the real, independently
disturbed process; $f$, $\sigma$, $c$ need only be evaluable in batch.

\paragraph{Quantile-vector state.} Fix levels $0 < q_1 < \dots < q_K < 1$
($\levels = \{0.05, 0.25, 0.5, 0.75, 0.95\}$ here). QPP replaces the $S$
particles of a sampled rollout by $K$ deterministic trajectories, one per
level, all started at $z_0$:
\begin{equation}
z^{(k)}_{t+1} = f\big(z^{(k)}_t, u_t\big) + \sigma\big(z^{(k)}_t, u_t\big)\, F^{-1}(q_k),
\qquad k = 1, \dots, K.
\label{eq:propagate}
\end{equation}
Trajectory $k$ reuses $q_k$ at every step and coordinate. Thus the $K$
states evaluate the comonotonic process $Z_t(U)$ (one uniform $U$ shared
across time and coordinates) at $U=q_k$. Set
$C_t^{(k)}=c(z_t^{(k)},u_t)$; each step costs $K$ model evaluations per
candidate.

\begin{figure}[t]
\centering
\includegraphics[width=\textwidth]{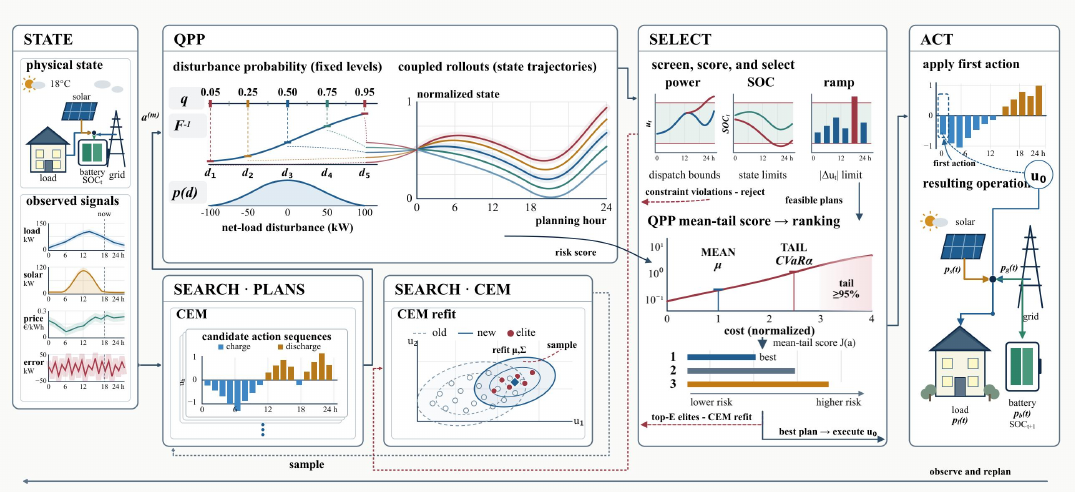}
\caption{QPP in the closed-loop planning pipeline. Fixed disturbance levels
generate coupled trajectories and an analytic mean--tail score; CEM ranks
candidate plans, executes the first action, observes the next state, and
replans.}
\label{fig:qpp-overview}
\end{figure}

\paragraph{Level-curve scoring.}
At each step, let $C_t(\cdot)$ linearly interpolate the points
$(q_k,C_t^{(k)})$. We extend it as a constant on $[0,q_1]$ and $[q_K,1]$,
and let $U\sim\mathrm{Unif}(0,1)$. Its mean is a trapezoid rule
whose end weights include the boundary masses; omitting them creates bias
that does not vanish as the interior grid is refined. We use $\alpha$ for the
\emph{upper-tail probability mass}, so $\CVaR_{0.1}$ is the mean of the
worst $10\%$. We use the representation
\[
  \CVaR_\alpha(X)
  =\min_v\left\{v+\alpha^{-1}\mathbb{E}\big[(X-v)^+\big]\right\}
\]
from \citet{rockafellaruryasev2000cvar}, because $(C_t-v)^+$ has an
analytic integral on every linear segment. The implementation brackets $v$
by the row minimum and maximum and uses $60$ bisection steps;
Appendix~\ref{sec:appendix-scoring} gives the weights and cell integrals. The
result is exact to floating-point precision for the constructed interpolant,
not for the unobserved curve between or beyond the knots. The cost row need
not be sorted: a non-monotone cost can make its largest value occur away from
the largest $q_k$, but integration still follows the level coordinate $u$.

\paragraph{Order restriction along $K$.} Interpreting trajectory $k$ as a
marginal $q_k$ quantile requires each coordinate of $Z_t(U)$ to be
non-decreasing in $U$ (Proposition~\ref{prop:exact}(c)). Pool-adjacent
violators~\citep{barlow1972pav,rwd1988pav} or
rearrangement~\citep{cfg2010rearrangement} can impose this order
(Appendix~\ref{sec:appendix-proofs}), but QPP uses neither: level-curve
scoring does not require ordered costs, and projecting a state would change
the constructed trajectory whenever the condition fails
(Appendix~\ref{sec:appendix-setup-detail}).

\paragraph{Objective.} A candidate $a = u_{0:H-1}$ is scored by
\begin{equation}
J(a) = \sum_{t=0}^{H-1} \Big[ (1 - \lam)\, \mathbb{E}[C_t(U)] + \lam\, \CVaR_\alpha[C_t(U)] \Big],
\qquad \lam \in [0, 1],
\label{eq:objective}
\end{equation}
lower is better: a sum of per-step tail risks, a conservative surrogate for
trajectory CVaR because subadditivity gives
\[
  \sum_t \CVaR_\alpha[C_t]
  \geq \CVaR_\alpha\!\left[\sum_t C_t\right].
\]
This surrogate uses the comonotonic coupling, not the independent process
executed (Sections~\ref{sec:theory}, \ref{sec:results}). At $\lam=0$ the
score is risk-neutral for the coupled
process, which differs from a mean rollout when $c$ is nonlinear; at
$\lam=1$ it is the summed per-step tail. The level-curve evaluator is the
central component, while $\lam$ is an optional decision preference.
Appendix~\ref{sec:appendix-review-sensitivity} separately audits a single
CVaR applied to the total trajectory cost.

Algorithm~\ref{alg:qpp-score} collects the evaluator for one candidate; the
planner evaluates candidates in batch. The same $q_k$ is retained throughout
each trajectory, and the cost row is neither sorted nor projected.
\begin{algorithm}[t]
\DontPrintSemicolon
\caption{QPP level-curve evaluation for one candidate action sequence.}
\label{alg:qpp-score}
\KwIn{State $z_0$, actions $u_{0:H-1}$, levels $\levels$, tail mass $\alpha$,
and weight $\lambda$}
\KwIn{Evaluable $f$, $\sigma$, and $c$, and disturbance quantile function
$F^{-1}$}
\KwOut{Score $J$ and propagated states $\{z_t^{(k)}\}$}
$J \gets 0$ and $z_0^{(k)} \gets z_0$ for $k=1,\ldots,K$\par
\For{$t=0,\ldots,H-1$}{
  \textbf{A. Propagate the fixed levels}\par
  \For{$k=1,\ldots,K$}{
    $C_t^{(k)} \gets c(z_t^{(k)},u_t)$\par
    $z_{t+1}^{(k)} \gets f(z_t^{(k)},u_t)
      + \sigma(z_t^{(k)},u_t)F^{-1}(q_k)$\par
  }
  \textbf{B. Build the level curve}\par
  Form the constant-extended piecewise-linear $C_t$ through
  $(q_k,C_t^{(k)})_{k=1}^K$\par
  \textbf{C. Evaluate and accumulate the score}\par
  $m_t \gets \int_0^1 C_t(u)\,du$ and
  $r_t \gets \CVaR_\alpha[C_t(U)]$ by analytic integration\par
  $J \gets J+(1-\lambda)m_t+\lambda r_t$\par
}
\textbf{return} $J$ and $\{z_t^{(k)}\}$\par
\end{algorithm}

\paragraph{Planner, baselines, cost.} The outer loop is the
cross-entropy method: draw $P$ candidates from a diagonal Gaussian, clip,
score, refit to the best $E$, and finally rescore the elite mean. Both the
sampler (Gaussian, iCEM-style~\citep{pinneri2020icem}, or
\dsCEMstyle{}~\citep{walker2026dscem}) and evaluator are modular. Monte
Carlo uses $S$ independent samples and averages the worst
$\lceil\alpha S\rceil$; mean rollout propagates $f$ alone; moment matching
uses a full central-difference Jacobian with diagonal covariance
(Appendix~\ref{sec:appendix-envs}). QPP uses $K$ model calls per
candidate-step and $O(KB)$ scoring for fixed $B=60$; the other costs are
$S$ calls plus a sort, one call, and $2d+1$ calls, respectively. Thus $S=K$
is the budget-matched QPP--Monte Carlo comparison, with $\Nevals$ denoting
the number of dynamics evaluations.
 \section{Theory}
\label{sec:theory}

$Z_t(U)$ denotes the \emph{coupled process} of Section~\ref{sec:method}:
\[
Z_{t+1}(U)=f(Z_t(U),u_t)+\sigma(Z_t(U),u_t)F^{-1}(U),\qquad Z_0=z_0.
\]
One uniform $U$ is shared across steps and coordinates; $\tilde Z_t$ is
the \emph{independent process}, a fresh $\varepsilon_t$ per step, on which
plans are executed. The guarantees below concern this scalar ordered
disturbance construction; they do not cover arbitrary multivariate or
non-monotone learned dynamics.

The next result separates three claims: exactness at the selected knots,
convergence as the grid is refined, and the stronger conditions under which
the propagated levels are marginal quantiles. It applies to the states and
level curves produced by Algorithm~\ref{alg:qpp-score}.

\begin{proposition}[Three guarantees for QPP]
\label{prop:exact}
\leavevmode
\begin{enumerate}
  \setlength{\itemsep}{2pt}
  \setlength{\parsep}{0pt}
  \setlength{\topsep}{2pt}
  \item[(a)] \textbf{Knot exactness.} For every action sequence, $t$, and
  $k$, $z_t^{(k)}=Z_t(q_k)$. Moreover, $C_t(U)$ is the piecewise-linear
  interpolation of $u\mapsto c(Z_t(u),u_t)$ at the knots $\levels$, with no
  ordering assumption.
  \item[(b)] \textbf{Grid-refinement convergence.} Suppose this cost map is
  continuous on compact subintervals of $(0,1)$ and integrable. If a sequence
  of grids has interior mesh tending to zero, endpoint knots tending to $0$
  and $1$, and uniformly integrable constant-extended interpolants, then the
  level-curve mean and $\CVaR_\alpha$ converge to those of
  $c(Z_t(U),u_t)$.
  \item[(c)] \textbf{Marginal-quantile interpretation.} If, moreover,
  $F^{-1}$ is locally absolutely continuous, $f$ and $\sigma$ are continuously
  differentiable, $\sigma\geq0$, and, for every $u\in(0,1)$ and every step,
  $J_f+F^{-1}(u)J_\sigma$ has non-negative entries along $Z_t(u)$, then each
  coordinate of $Z_t(U)$ is non-decreasing in $U$. Thus
  $z_{t,i}^{(k)}$ is the $q_k$ quantile of $Z_{t,i}(U)$, and the order
  restriction of Section~\ref{sec:method} is a no-op.
\end{enumerate}
\end{proposition}

The diagnostic in Appendix~\ref{sec:appendix-setup-detail} checks only the
weaker discrete statement that the $K$ propagated trajectories are ordered
along $K$ at every step.

\begin{figure}[t]
\centering
\includegraphics[width=0.96\textwidth]{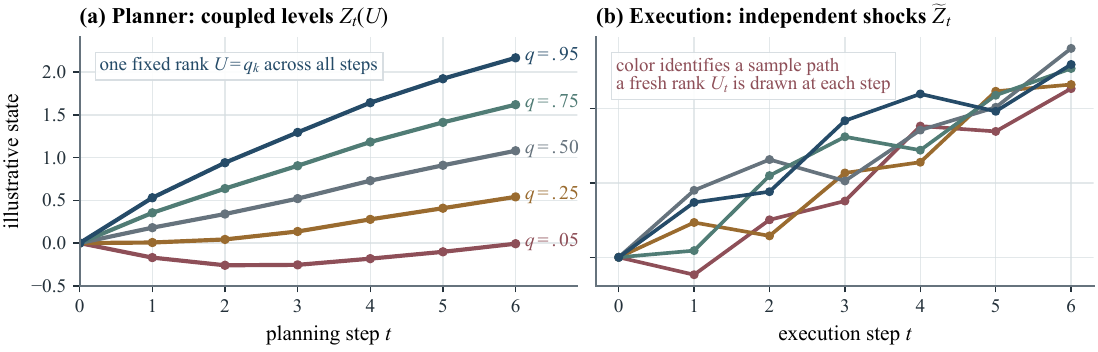}
\caption{Planning holds one rank through the horizon (left; colored endpoint
labels give the five fixed quantile levels), whereas execution draws a fresh
rank at each step (right; colors only distinguish sample paths). Equal
one-step marginals need not imply equal temporal dependence or tail behavior.}
\label{fig:coupling-objective}
\end{figure}

The identities are exact at the selected knots. Against a $100{,}000$-sample
reference, the largest relative errors are $0.0086$ after one step and
$0.0022$ after ten (Appendix~\ref{sec:appendix-setup-detail}). These values do
not measure error between knots: on the monotone-map diagnostic, relative
level-curve interpolation error is $0.1100$ at the deployed $K=5$
($0.1344$/$0.0978$ at $K=3$/$9$). The sign condition holds for overload and
the CityLearn surrogate, but fails for inventory and the rotation map. The
microgrid is order-preserving in the shared level, including at charge
clipping; because clipping is not continuously differentiable, it lies
outside the proposition's smooth sufficient condition
(Section~\ref{sec:setup}; measured interpolation error at most $0.00117$).

\begin{remark}[The spread law]
\label{rem:spread}
For $z' = a z + s\,\varepsilon$ with $0 \le a < 1$, the ratio of the
standard deviation of $Z_H$ under the coupling to that of $\tilde Z_H$ is
\begin{equation}
\Rspread(H, a) = \sqrt{\frac{(1 - a^H)(1 + a)}{(1 + a^H)(1 - a)}},
\label{eq:spread}
\end{equation}
equal to $1$ at $H = 1$, increasing in $H$ and $a$, saturating at
$\sqrt{(1 + a)/(1 - a)}$: the coupling adds shocks coherently, independence
in quadrature.
\end{remark}

The measured gap between propagated and independent-process quantiles on
the Gaussian monotone-diagonal map is monotone in $H$ as~\eqref{eq:spread}
predicts and falls inside the band the closed form implies at every
tested $H$ and persistence (Appendix~\ref{sec:appendix-setup-detail} gives
the values and the simulation check). The closed form is for the
constant-loading scalar case: an analogy for overload (action-scaled
loading) and inapplicable to inventory's stock (a state-dependent loading;
Section~\ref{sec:results-mean} reports the measured ratio). At the
microgrid's fitted persistence and $H = 24$ the ratio is within a percent
of its saturating value, the analytic reason the coupled score there is a
loose upper bound (Section~\ref{sec:results-microgrid}).

\begin{proposition}[One-sided conservatism]
\label{prop:conservative}
Let the drift be linear with non-negative coefficients and the loadings
non-negative constants, so that each $Z_{t,i}$ is a non-negative linear
combination of $\varepsilon_1, \dots, \varepsilon_t$ plus a deterministic
term, and let every per-step cost be a convex non-decreasing function of a
non-negative linear combination of the state (a ridge cost), or a sum of
such. Assume also that the horizon-summed costs and their stop-loss
transforms have finite expectations under both couplings. Then the
horizon-summed cost is larger in the increasing convex
order under the coupling than under independence; in particular
$\CVaR_\alpha[\sum_t c(Z_t(U), u_t)] \ge \CVaR_\alpha[\sum_t c(\tilde Z_t,
u_t)]$ for every $\alpha$.
\end{proposition}

The proof composes supermodularity of ridge costs and their increasing convex
transforms with the extremal property of comonotonic coupling
~\citep{kaasdhaenegoovaerts2000comonotonic}; CVaR respects the resulting
increasing convex order. Together with subadditivity, this makes the exact
$\lam=1$ summed-tail objective an upper bound under the stated hypotheses.
The finite interpolant has no general sign, however, so the implemented
score is not a formal certificate without a certified envelope. The
hypotheses also matter: a concave rotation reverses the inequality
($-0.855$ versus $-0.647$), and inventory receives no certificate. Overload
has a measured coupled-minus-executed $\CVaR_{0.1}$ of $+21.4$, while the
microgrid gives $302.88$ versus $8.82$ in the same conservative direction
(Appendix~\ref{sec:appendix-setup-detail}). No rate is claimed for nonlinear
drifts; monotone-systems theory supplies a sign condition, not a rate
~\citep{angelisontag2003monotone}. Because the microgrid clips charge, its
ordering remains empirical rather than a by-construction proof.
 \section{Experimental Setup}
\label{sec:setup}

\paragraph{Environments.} Three environments with asymmetric tails carry
the claims: two closed-form and one built from real load data (full
definitions are in Appendix~\ref{sec:appendix-envs}). \emph{Overload} has
one state, $H=8$, and an action-scaled noise tail that trades throttle reward
against squared-hinge overload cost. \emph{Inventory} has $H=12$, persistent
lognormal demand ($\rho=0.8$), and stockout cost ten times holding cost.
Synthetic maps provide calibration and negative controls; a fitted CityLearn
surrogate is used only behind its validity gates.

\paragraph{Microgrid.} The third environment is battery dispatch against a
real hourly load trace. Its state is battery charge, an AR(1) standardized
log-load anomaly and hour; $H=24$ and the action is a charge/discharge
fraction in $[-1,1]$. A pooled training-client table supplies innovations;
site-specific quantities are fitted from the first $70\%$ of rows. Twelve
clients fit the pooled AR/noise model, while four retain local calibration.
Because the latter also enter the D1 design screen, they are not untouched:
sitewise evidence is conditional (Appendix~\ref{sec:appendix-gates}). The
state is ordered in the shared level through charge clipping, although
clipping lies outside Proposition~\ref{prop:exact}'s smooth condition. The
prospective endpoint $V1$ is insufficient headroom for the next full-power
discharge, not actual import above capacity; that endpoint is a separately
labelled post-hoc replay.

\paragraph{Planners and budget.} Every planner uses CEM with $10$ iterations,
$10$ elites and clipped candidates. QPP ($K=5$, $\alpha=0.2$,
$\lam\in\{0,0.5,1\}$) and Monte Carlo ($S=5$) use $P=64$; mean rollout uses
$PK=320$; moment matching is within $5\%$ of QPP's dynamics-evaluation
budget. $\Nevals$ includes final elite-mean scoring. A control tests the
different elite fractions (Appendix~\ref{sec:appendix-setup-detail});
composition studies otherwise vary only the sampler and loop refinements.
The microgrid also uses a PETS-style TS-$\infty$ bootstrap ensemble matched at
$25$ calls per candidate-step and evaluated by held-out trace replay.

Planning optimizes a $20\%$ per-step tail surrogate, whereas evaluation
reports the worst $10\%$ of realized episode costs; both choices were frozen,
and are not asserted equivalent. An exploratory audit separately tests
$\alpha=0.1$, total-trajectory CVaR and terminal alignment
(Appendix~\ref{sec:appendix-review-sensitivity}).

\paragraph{Common random numbers and statistics.} Repeat $r$ gives every
planner the same noise streams, enabling paired contrasts. Confirmatory
tables use $n=200$ ($100$ for $\lam$ and replanning sweeps); the corrected
inventory count audit uses $n=400$. Empirical $\CVaR_{0.1}$ averages the
worst $20$ realized episode costs over optimizer and execution randomness,
not fixed-plan conditional risk. Intervals use paired CRN bootstraps; mean,
CVaR and violation rows use sign-flip, stored bootstrap-tail and exact
McNemar procedures. Each site table is one Holm family. The four frozen
sitewise bounds are unadjusted and descriptive; a labelled post-hoc audit
adds Holm tests and simultaneous Bonferroni bounds. Prospective rules and
test details are in Table~\ref{tab:gates} and
Appendix~\ref{sec:appendix-gates}; the budget is model calls, not runtime or
memory.
 \section{Results}
\label{sec:results}

Conventions are those of Section~\ref{sec:setup}; $\CVaR_{0.1}$ is the
upper $10\%$ tail of the realized cost over repeats (lower is better).

\begin{figure}[t]
\centering
\includegraphics[width=\textwidth]{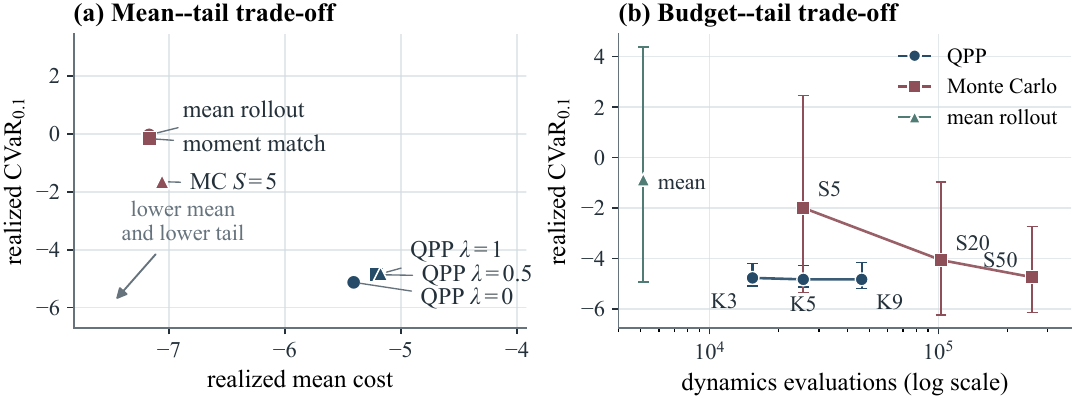}
\caption{Controlled-environment results on overload. (a) QPP moves along a
mean--tail frontier rather than dominating every metric. (b) Increasing
trajectory count improves sampled evaluation, but the low-budget QPP points
remain competitive; bars show bootstrap intervals. The appendix reports the
full statistical comparisons and inventory counterexample.}
\label{fig:controlled-results}
\end{figure}

\subsection{Tail risk against noise-blind planning}
\label{sec:results-mean}

\begin{table}[h!]
\centering
\small
\begin{tabular}{lrrrrr}
\toprule
& \multicolumn{3}{c}{$\CVaR_{0.1}$} & \multicolumn{2}{c}{Mean cost} \\
\cmidrule(lr){2-4}\cmidrule(l){5-6}
Environment & QPP & Mean & $\Delta$ & QPP & Mean \\
\midrule
Overload & -4.858 & -0.032 & -4.826 & -5.209 & -7.170 \\
Inventory & 163.139 & 270.013 & -106.874 & 141.384 & 63.146 \\
\bottomrule
\end{tabular}
 \caption{Point estimates against mean rollout; lower is better and $\Delta$
is QPP minus mean rollout. Full results are in Appendix
Tables~\ref{tab:overload-mean-full} and~\ref{tab:inventory-mean-full}.}
\label{tab:results-noise-blind}
\end{table}

Table~\ref{tab:results-noise-blind} summarizes the main trade-off. On overload
($H=8$; Figure~\ref{fig:controlled-results}a), QPP at $\lam=0.5$ lowers the
tail cost and reduces the overload rate from $2.5\%$ to $0.5\%$, but its mean
cost is worse. The frozen B1 bound is met, so the result is a mean--tail
frontier rather than overall domination. A control with the same elite
fraction as mean rollout changes the tail by less than one tenth of the QPP
gaps on overload and inventory; full statistical audits are in
Appendix~\ref{sec:appendix-setup-detail}.

On inventory ($H=12$, $\rho=0.8$), QPP records no stockout in $200$ episodes,
whereas mean rollout has a rate of $0.670$. This tail improvement comes with
a large mean-cost penalty. Moreover, $S=5$ Monte Carlo has a lower tail cost
($145.252$), and QPP with $\lam=0$ reaches $144.972$. Thus the nonzero risk
weight is not the source of the inventory benefit.

\subsection{Against budget-matched sampling}
\label{sec:results-mc}

\begin{table}[h!]
\centering
\small
\begin{tabular}{lrrrrr}
\toprule
Environment & $n$ & QPP $K=5$ & MC $S=5$ & MC $S=20$ & MC $S=50$ \\
\midrule
Overload & 200 & -4.823 & -1.989 & -4.058 & -4.735 \\
Inventory & 400 & 163.558 & 208.987 & 130.009 & 113.208 \\
\bottomrule
\end{tabular}
 \caption{Realized $\CVaR_{0.1}$ in the trajectory-count sweeps; lower is
better. Inventory reports the corrected $n=400$ audit; full statistical results and
extended sweeps are in the appendix.}
\label{tab:results-budget}
\end{table}

Table~\ref{tab:results-budget} separates point estimates from interpretation.
At the matched evaluator budget ($S=K=5$), the overload gap for QPP at
$\lam=0.5$ is $-3.210$, but the adjusted test does not support a tail-cost
difference ($p_{\mathrm{adj}}=0.6183$). The mean-cost penalty is $1.849$
($p_{\mathrm{adj}}=0.0015$). The inventory matched-budget tail comparison is
also unresolved (Table~\ref{tab:qmc-inventory}), and QPP loses the mean
comparison decisively.

The trajectory-count sweeps make the environment difference clear. On
inventory, $S=20$ Monte Carlo is already below $K=5$ QPP, so B3 is not met.
On overload, no $S\le50$ arm comes within $5\%$ of QPP; the extension to
$S=200$ is unresolved relative to $S=50$, and a clipped coupled sampler first
wins at $S=100$. The evidence therefore supports a low-budget overload
advantage, not a general advantage over sampling.

On the fitted microgrid no IID arm comes within $5\%$ of QPP through
$S=200$; IID tail cost remains $9$--$15\%$ worse and a clipped coupled
sampler leaves a $5$--$8\%$ gap. Population, elites and iterations are held
fixed, so this is not a best-allocation total-budget comparison
(Appendix~\ref{sec:appendix-microgrid-crossover}).

\subsection{A real-data tail environment}
\label{sec:results-microgrid}

\paragraph{Calibration and H1.} Four locally calibrated clients also enter
D1, so their evidence is conditional; pooled calibration passes only $3/8$
split seeds, and two planned sites exceed the PIT tolerance. At $n=200$ and
$H=24$, all planners meet the budget gates.
Figure~\ref{fig:microgrid-gates} shows improved $V1$ at all sites, B1 passing
at $242$, $253$ and $279$, and B2 failing everywhere. A post-hoc replay gives
actual-import rates of $28.5$--$37.5\%$ for QPP versus $43.0$--$62.5\%$ for
mean rollout. This supports reduced incidence, not safety
(Appendix~\ref{sec:appendix-microgrid-h1}). A Holm audit rejects the $10\%$
B1 null at three sites, but Bonferroni retains only $242$ and $279$. At
$n=1{,}000$, neither $\alpha=0.1$ nor terminal alignment helps consistently;
breaking rank sharing worsens tail cost and V1
(Appendix~\ref{sec:appendix-review-sensitivity}).

\begin{figure}[t]
\centering
\includegraphics[width=0.96\textwidth]{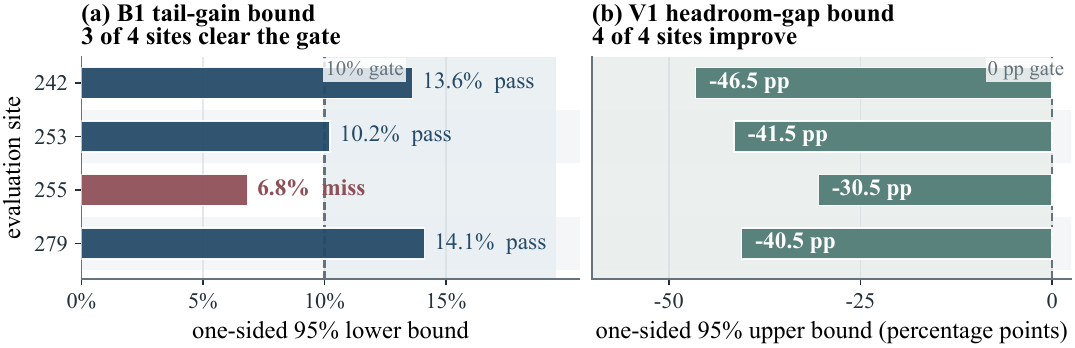}
\caption{Sitewise one-sided $95\%$ bounds for the fitted microgrid. Dashed
lines mark the gates: B1 passes at three sites, while V1 improves at all four.
Bounds are frozen and unadjusted; the multiplicity audit is in the appendix.}
\label{fig:microgrid-gates}
\end{figure}

\label{sec:results-ensemble}
\label{sec:results-composition}
\label{sec:results-lambda}
 \section{Limitations and Scope}
\label{sec:negatives}

\paragraph{Measurement.} Model-call matching omits runtime and memory;
decision-variance reduction is unsupported except through inventory
saturation. Other limits are loose coupling, $11\%$ interpolation error at
$K=5$, an unpriced terminal state, and 20 tail outcomes per $n=200$ record.

\paragraph{Selection.} D1 screens all sites; the construction fails D1 on a
second dataset~\citep{miller2020bdg2}, and KPI calibration and CityLearn
replanning fail. Proposition~\ref{prop:conservative} does not sign the finite
interpolant, and a concave rotation reverses its ordering. Evidence does not
establish untouched-site, learned-model, epistemic, fixed-plan or
decision-variance performance, nor an unconditional budget advantage.
 \section{Conclusion}
\label{sec:conclusion}

QPP integrates deterministic comonotonic disturbance levels. It is most useful
when planning budgets are tight and shared-rank tail structure is informative.
Its benefit can disappear when coupling or model fit is weak.
 
\subsection*{Generative AI Use}
Generative AI tools were used to assist with writing and editing, figure
refinement, and code review. The authors reviewed and verified all AI-assisted
content and take responsibility for the final content.\label{maintext:end}

\appendix
\section{Additional proofs and experimental details}
\label{sec:appendix}

\subsection{Environment definitions}
\label{sec:appendix-envs}

This section first defines the controlled environments and their baselines,
then gives the fitted-environment construction used in the application
studies. The final part defines the microgrid endpoints and ensemble
evaluators. This order separates the dynamics and costs from the data split
and the later evaluation rules.

\subsubsection{Controlled environments and baselines}

The \emph{overload} environment has one state and horizon $H=8$. Its
dynamics are $z'=az+gu+(s_0+s_1u)\varepsilon$, where $\varepsilon$ is
standardized lognormal noise and the throttle satisfies $u\in[0,1]$. The
pre-step cost is $-ru+c\,\mathrm{relu}(z-\theta)^2$. The constants are
$a=0.7$, $g=0.5$, $s_0=0.05$, $s_1=0.5$, $\theta=2.5$, $c=8$, $r=1$, and
$z_0=0$. The deterministic path never crosses $\theta$
for any constant throttle; the action-scaled tail does. The constants were
tuned by a $200{,}000$-sample search for that separation; this is a disclosed
design choice rather than an independently validated property.

The \emph{inventory} state is $[x,m]$, representing stock and a smoothed
demand level, and the horizon is $H=12$. The dynamics are
$x'=x+cu-m-s_Dm\varepsilon$ and
$m'=\rho m+(1-\rho)\bar m+s_mm\varepsilon$, with one shared standardized
lognormal shock per step and order fraction $u\in[0,1]$. The pre-step cost is
$h\,\mathrm{relu}(x)+p\,\mathrm{relu}(-x)+c_ou$. The constants are $h=1$,
$p=10$, $c_o=0.5$, $\rho=0.8$, $\bar m=1$, $c=3$, $s_D=0.5$, $s_m=0.1$,
and $z_0=[1,1]$. A step is a stockout when the post-step $x<0$. No
constant was tuned against a planning outcome; the mean- and tail-optimal
constant orders come from a $200{,}000$-sample grid search under the
implemented cost timing ($0.41$ and $0.53$).

The \emph{synthetic} maps include monotone diagonal dynamics (per-coordinate persistence
$a_i \sim \mathrm{Unif}(0.5, 0.9)$), dense $\tanh$, and rotation dynamics
with Gaussian, lognormal, or Student-$t_3$ standardized noise. The
monotone-diagonal family with a symmetric quadratic cost is the planning
negative control.

For moment matching, central differences estimate the full Jacobian. The
baseline propagates marginal variances while treating the covariance as
diagonal. With state mean $\mu$, marginal variance $v$, basis vector $e_j$,
and difference step $\epsilon$, the update is
\[
\mu'=f(\mu,u),\qquad
J_{ij}=\frac{f_i(\mu+\epsilon e_j,u)-f_i(\mu-\epsilon e_j,u)}{2\epsilon},
\qquad
v_i'=\sum_{j=1}^{d}J_{ij}^2v_j+\sigma_i(\mu,u)^2.
\]
It therefore uses $2d+1$ dynamics evaluations per step: each perturbed call
returns every output coordinate. The method transmits marginal uncertainty
across state coordinates but does not carry cross-covariances between steps;
it is a first-order diagonal-covariance approximation, not an exact Gaussian
filter.

\subsubsection{CityLearn surrogate}

The CityLearn surrogate is a diagonal linear state-space model whose dynamics
are monotone. Its noise table contains empirical residual quantiles. The model
is fitted to $2{,}000$ random transitions of the real simulator. Planning uses
a horizon of $H=24$ hours on the surrogate, and execution uses the real
environment. The validity gates require cost $R^2>0$ and multi-step quantile
calibration error below $0.10$.

\subsubsection{Microgrid dynamics and cost}

The microgrid has state $[\mathrm{soc},\ell,h]$, action $u\in[-1,1]$, and
horizon $H=24$. The state contains the charge fraction, a standardized AR(1)
log-load anomaly, and an unbounded hour counter. Positive $u$ charges and
negative $u$ discharges.
Writing $u_+=\max(u,0)$ and $u_-=\max(-u,0)$, the transition is
\[
\begin{split}
\mathrm{soc}'&=\operatorname{clip}\!\left(
\mathrm{soc}+\frac{\eta_c u_+p_{\mathrm{bat}}\Delta t}{e_{\mathrm{cap}}}
-\frac{u_-p_{\mathrm{bat}}\Delta t}{\eta_de_{\mathrm{cap}}},\,0,\,1\right),\\
\ell'&=\rho\ell+\sqrt{1-\rho^2}\,\varepsilon,\qquad h'=h+1 .
\end{split}
\]
Here $\eta_c=\eta_d=0.95$, $e_{\mathrm{cap}}=2.0$,
$p_{\mathrm{bat}}=0.5$, and $\Delta t=1.0$. Let
$\Delta\mathrm{soc}=\mathrm{soc}'-\mathrm{soc}$. The bus-side battery draw,
positive when charging, is
\[
B=\begin{cases}
\Delta\mathrm{soc}\,e_{\mathrm{cap}}/(\eta_c\Delta t),
&\Delta\mathrm{soc}\geq0,\\
\Delta\mathrm{soc}\,e_{\mathrm{cap}}\eta_d/\Delta t,
&\Delta\mathrm{soc}<0.
\end{cases}
\]
Normalized demand and net grid import are
\[
L_{\mathrm{pu}}=\exp(\mathrm{profile}(h)+s_r\ell),\qquad
G_{\mathrm{pu}}=L_{\mathrm{pu}}+B,
\]
where $G_{\mathrm{pu}}<0$ denotes export and $\mathrm{profile}$ is a
two-harmonic Fourier curve. If $s_{\mathrm{site}}$ is the reported site
scale, then $L_{\mathrm{physical}}=s_{\mathrm{site}}L_{\mathrm{pu}}$ and
$G_{\mathrm{physical}}=s_{\mathrm{site}}G_{\mathrm{pu}}$. Thus
$p_{\mathrm{cap}}$, $p_{\mathrm{bat}}$ and $e_{\mathrm{cap}}$ are in the
normalized system's per-unit power and per-unit-hour coordinates; the tariff
below prices per-unit energy rather than claiming a physical currency bill.
With $[x]_+=\max(x,0)$, the pre-step cost is
\[
c(z,u)=\pi(h)[G_{\mathrm{pu}}]_+-p_{\mathrm{exp}}[-G_{\mathrm{pu}}]_+
+w_o[G_{\mathrm{pu}}-p_{\mathrm{cap}}]_+^2+w_d u_{\mathrm{eff}}^2,
\]
where $u_{\mathrm{eff}}$ is the realized action after state-of-charge
saturation. The flat import tariff is $\pi=0.20$, export credit
$p_{\mathrm{exp}}=0.04$, the overload-penalty ratio is $10$, and degradation
is $0.10$ per pu-h. Thus $\sigma=[0,\sqrt{1-\rho^2},0]$ is constant and
non-negative.

\subsubsection{Microgrid data, split, and fitted quantities}

The re-dated ECL hourly mirror contains $26{,}304$ rows for $321$ clients from
2016-07-01 to 2019-07-02, matching the original release
span~\citep{trindade2015ecl}. Of these clients, $92$ are strictly positive on
every row. The reproducible seed-0 cohort contains clients 127, 130, 137, 189,
197, 202, 220, 225, 233, 240, 242, 253, 255, 257, 270, and 279. Twelve form the
pooled training set; sites 242, 253, 255, and 279 are excluded from pooled
fitting and used for evaluation.

Each site's scale, profile, $s_r$, and $p_{\mathrm{cap}}$ are fitted from its
first $70\%$ of rows; the final $30\%$ is reserved for evaluation.
$p_{\mathrm{cap}}$ is the site's empirical $0.95$ quantile of unshaved
training load. The standardized innovation table is fitted only from the
twelve training clients and is cut at the $0.1$ and $99.9$ percentiles. It
contains $220{,}932$ values, has endpoints $-4.0116$ and $+5.1985$, raw mean
$0.00009$, and standard deviation $0.99996$. The fitted persistence is
$\rho=0.7525$, with noise loading $0.6586$.

\begin{figure}[!b]
\centering
\includegraphics[width=0.98\textwidth]{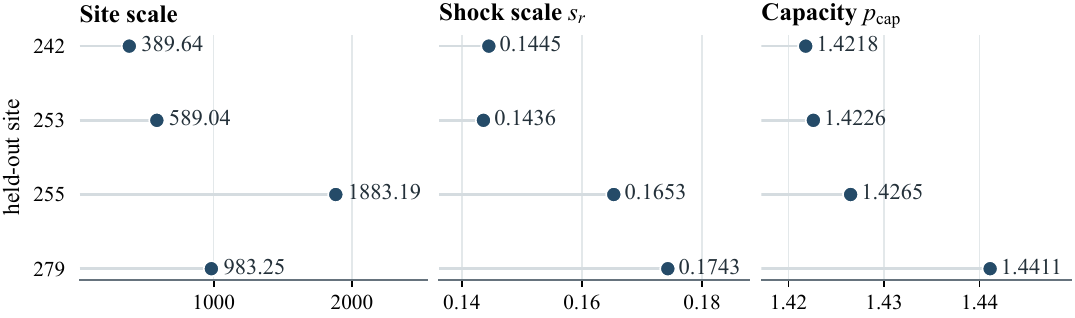}
\caption{Site-specific quantities fitted without pooling the four held-out
evaluation sites. Labels give the exact scale, innovation scale $s_r$, and
capacity threshold $p_{\mathrm{cap}}$ used by each site's simulator. The
separate axes preserve the three quantities' natural units rather than
normalizing them to a common visual scale.}
\label{fig:microgrid-site-parameters}
\end{figure}

Two disclosed design revisions preceded freezing the D1 rule: a flat tariff
replaced an originally planned two-level tariff, and the innovation table was
cut at its $0.1\%$ ends. Both revisions make the tail smaller, working against
a tail-aware planner rather than in its favor.

\subsubsection{Microgrid endpoints and ensemble evaluators}

For the primary $V1$ endpoint, define
\[
D_{\max}(s)=\min\!\left(p_{\mathrm{bat}},
\frac{s e_{\mathrm{cap}}\eta_d}{\Delta t}\right),\qquad
h(z)=L_{\mathrm{pu}}(z)-D_{\max}(\mathrm{soc})-p_{\mathrm{cap}},
\]
\[
V1=\mathbf 1\!\left\{\max_{1\leq t\leq H}h(z_t)>0\right\}.
\]
It is evaluated on every post-transition state, including $z_H$. Hence it
marks loss of discharge headroom, not necessarily realized overload under the
action already taken. The separate exploratory event is
$\mathbf 1\{\exists t\in\{0,\ldots,H-1\}:G_{\mathrm{pu}}(z_t,u_t)>
p_{\mathrm{cap}}\}$.
The exact-stream replay uses the same execution seed sequence as the recorded
cost and $V1$ observations. Rates for every planner, paired bounds and the
cross-site multiplicity policy are reported with the H1 tables below. Because
this endpoint was added after the frozen protocol, it remains an exploratory
operational diagnostic.

\paragraph{Ensemble evaluators.} Each bootstrap member resamples the
training interval in $168$-hour blocks with replacement and refits the
site models. Its pooled AR coefficient and innovation table use a
without-replacement $75\%$ subsample of the twelve training clients;
$p_{\mathrm{cap}}$ remains the reference site's fixed per-unit threshold.
For the trajectory ensemble with $M=5$ and $K=5$, the same five levels are
propagated separately through each of five members, giving an $M\times K=25$
equal-weight vertical mixture for mean and CVaR scoring. The single-model QPP
arm with $M=1$ and $K=25$ propagates 25 levels through the reference model.
The PETS-style ensemble with $M=5$ and $S=5$ draws five IID trajectories from
each of five members and keeps each particle's member fixed throughout the
horizon (TS-$\infty$); its 25 realized trajectories are scored as ordinary
equal-weight Monte Carlo samples.

\subsection{Level-curve scoring details}
\label{sec:appendix-scoring}
For $K>1$, the constant-extended trapezoidal mean is
$\sum_{k=1}^K w_k C_k$, where
\[
w_1=q_1+\frac{q_2-q_1}{2},\qquad
w_k=\frac{q_{k+1}-q_{k-1}}{2}\ (1<k<K),\qquad
w_K=1-q_K+\frac{q_K-q_{K-1}}{2}.
\]
The weights sum to one; for $K=1$, $w_1=1$. For CVaR, the implementation
brackets the threshold $v$ by the row minimum and maximum and performs exactly
$60$ bisection iterations. On a sloped cell $[q_k,q_{k+1}]$, it computes the
crossing fraction
$r=\operatorname{clip}((v-C_k)/(C_{k+1}-C_k),0,1)$ and integrates the portion
where $C(q)>v$ by the exact trapezoid area. Increasing and decreasing cells
select opposite sides of the crossing; a flat cell is included wholly or not
at all. The two constant boundary cells are handled by the same rule.

After bisection, let $m_>$ be the measure of $\{q:C(q)>v_\star\}$ and $I_>$
the integral of $C$ on that set. The returned value is
\[
\frac{I_>+v_\star(\alpha-m_>)}{\alpha},
\]
which supplies the required fraction of any atom at $v_\star$. This is why
unsorted level rows and thresholds crossing either direction are handled
without ranking the values.

\subsection{Gate statements and outcomes}
\label{sec:appendix-gates}

This section is a registry of the paper's decision rules. It first reports the
outcomes, then defines the controlled-environment, composition, and microgrid
gate families in that order. A fraction such as $3/4$ means that three of
four site- or row-specific outcomes have the displayed status; it is not a
rate computed from the underlying episodes.

\begin{table}[h]
\centering
\footnotesize
\setlength{\tabcolsep}{3pt}
\begin{tabular}{@{}p{0.49\textwidth}cccc@{}}
\toprule
Gate & Over. & Inv. & Micro. & Other \\
\midrule
1a/b/c propagation versus coupled MC & & & & Sat. \\
D1 reserve discrimination $\ge 0.15$ & & & Sat. & Not sat. 2/2 \\
$G_{E2}$ calibration (skill, PIT, support) & & & Sat. & \\
B1 tail below mean rollout & Sat. & Sat. & Sat. 3/4 & \\
V1 headroom-loss rate below mean rollout & & & Sat. 4/4 & \\
H1 tail gap versus both noise-blind baselines $\ge 2$ & Sat. & & & \\
Surrogate validity: cost $R^2>0$, calibration $<0.10$ & & & & Sat. \\
B2 decision variance $\le 1/3$ of MC & Not sat. & Sat.$^\dagger$ & Not sat. 4/4 & Not sat. \\
B2eff planned-reserve variance (exploratory) & & & Sat. 1/4 & \\
B3 fixed-search trajectory-count sweep & N/A & Not sat. & Sat. 4/4 & \\
Coupled sampler within $5\%$ by $10\times$ & Inc. & & Not sat. 4/4 & \\
G2 surrogate gain $\ge 0.03$ & & & & Not sat. \\
G2-U3 replanning gain, five-way rule & & & & Negative \\
KPI simulator calibration & & & & Not sat. \\
E1 ensemble tail below budget-matched PETS & & & Sat. 3/4 & \\
E1v ensemble safer than PETS & & & Sat. 4/4 & \\
$G_{E3}$ source admissibility & & & Sat. & \\
\bottomrule
\end{tabular}
\caption{Primary, calibration and application criteria. ``Sat.'' means
satisfied, ``Not sat.'' means not satisfied, ``Inc.'' means inconclusive,
``Negative'' denotes a discharged negative result, and ``N/A'' denotes a
criterion that was not evaluated. $^\dagger$Satisfied by action-bound saturation and not claimed
(Section~\ref{sec:negatives}). Microgrid planning gates are read per site;
$G_{E2}$ and $G_{E3}$ are pooled cross-client source gates. A fraction reports
the number of site- or row-specific outcomes with the displayed status.}
\label{tab:gates}
\end{table}

\begin{table}[h]
\centering
\footnotesize
\setlength{\tabcolsep}{3pt}
\begin{tabular}{@{}>{\raggedright\arraybackslash}p{0.38\linewidth}
  *{4}{>{\centering\arraybackslash}p{0.13\linewidth}}@{}}
\toprule
Gate or diagnostic & Over. & Inv. & Micro. & Other \\
\midrule
C1 tail below MC under each sampler (4) & Not sat. 4/4 & Not sat. 4/4 & & \\
C2 tail below mean rollout under each sampler (4) & Sat. 4/4 & Sat. 4/4 & & \\
C3 tail gap within Gaussian-row margin (4) & Sat. 2/4; Inc. 2/4 & Not sat. 2/4; Inc. 2/4 & & \\
C4 candidate-stream variance halved (8) & Sat. 3/8 & Not sat. 8/8 & & \\
\midrule
\multicolumn{5}{@{}l}{\emph{Diagnostics without a prospectively specified bar}} \\
B2$'$ both-stream decision variance versus MC & & & & Not sat. \\
H1$'$ tail and mean below mean rollout & & & & Not sat. \\
E2/E3 epistemic decomposition/premium & & & Not sat. 4/4; Not sat. 4/4 & \\
E4 shift amplifies premium & & & N/A & \\
\bottomrule
\end{tabular}
\caption{Composition gates and ungated diagnostics, using the status labels
defined in Table~\ref{tab:gates}. C1--C4 were prospectively specified; the
bottom block was not.}
\label{tab:gates-composition}
\end{table}

\FloatBarrier
\subsubsection{Controlled-environment gates}

\begin{figure}[!b]
\centering
\includegraphics[width=\textwidth]{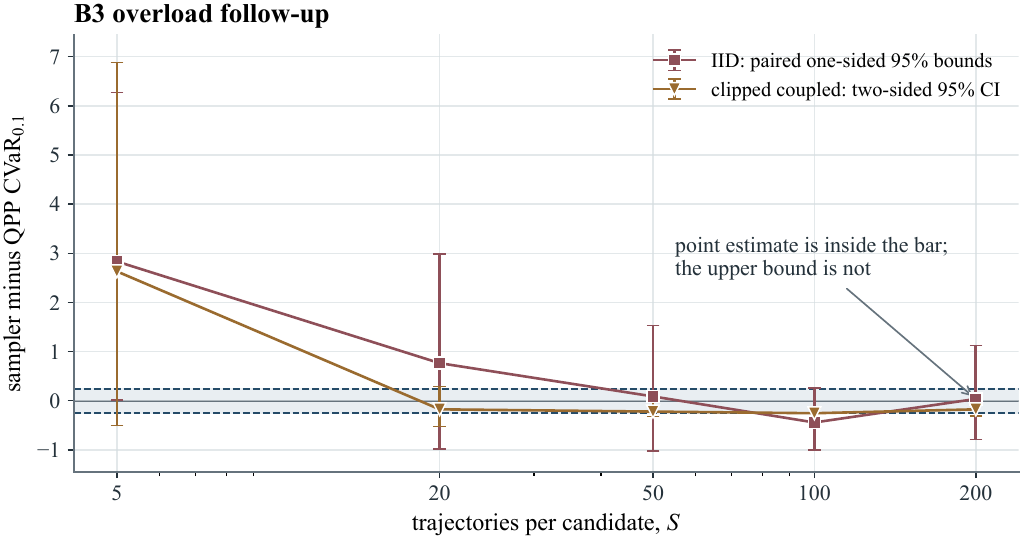}
\caption{Overload B3 trajectory-count follow-up. The curves show
sampler-minus-QPP tail-cost gaps. IID intervals join the paired one-sided
$95\%$ lower and upper bounds used by the B3 branches; clipped-coupled
intervals are two-sided $95\%$ intervals for the separate equivalence
diagnostic. The shaded region is the $\pm0.2411$ cost-unit comparison band.}
\label{fig:controlled-gate-guide}
\end{figure}

\paragraph{Accuracy and primary comparisons.}
Gates 1a--1c compare propagated quantiles with a coupled Monte Carlo
reference. B1 uses a one-sided tail-cost comparison: an absolute signed gap
on overload and a relative gain on inventory. H1 tests the overload gap
against both noise-blind baselines. G2 checks surrogate skill and calibration,
while B2 compares decision variance with budget-matched Monte Carlo.

\paragraph{Trajectory-count comparison.}
B3 asks for the smallest Monte Carlo sample count within $5\%$ of QPP at
$K=5$, with the CEM search fixed. That count must be at least $50$, or ten
times the per-candidate evaluator cost. The recorded rule labels the branch
in which no $S\leq50$ enters the band as \emph{not measured}, not met, and
the paper retains that verdict.

The overload follow-up extends the grid to $S\in\{5,20,50,100,200\}$ with
the same candidate seeds; the $S\leq50$ arms reproduce the closed record.
The absolute $5\%$ bar is $0.2411$ cost units. No IID arm qualifies because
its one-sided upper bound stays above the bar. At $S=200$, the point estimate
$0.0358$ is below the bar, but the upper bound is $1.1286$; the lower bound is
$-0.7891$, so neither decisive B3 branch applies and the recorded verdict
remains \emph{not measured}. This distinction between the estimate and the
gate bound is visible in Figure~\ref{fig:controlled-gate-guide}.

The clipped coupled sampler beats QPP at $S=100$, with a difference of
$-0.2545$ and a $95\%$ interval from $-0.3507$ to $-0.1472$
(Holm-adjusted $p=0.001$). At $S=200$, its difference is $-0.1752$ and the
interval runs from $-0.3071$ to $0.0466$, so the apparent advantage is no
longer conclusive.

\subsubsection{Surrogate and composition gates}

G2 requires a lower bound of at least $0.03$ on the relative
$\CVaR_{0.1}$ gain over mean rollout on the CityLearn surrogate. G2-U3 uses
a five-way decision at the strongest disturbance: met, discharged negative,
significant gain below the bar, inconclusive, or not measured. The KPI gate
checks simulator skill, PIT coverage, and support.

For each of four gated sampler rows, C1 and C2 test QPP against $S=5$ Monte
Carlo and mean rollout, respectively, using the B1 tail-cost rules. C3 asks
whether the row's QPP-minus-Monte-Carlo gap lies within the prespecified
margin around the Gaussian row's gap. Because that margin came from an
earlier exploratory gap rather than an application tolerance, C3 measures
internal stability, not practical equivalence. C4 requires the sampler's
candidate-stream decision variance to be at most half the Gaussian
sampler's. B2$'$ and H1$'$ are recorded diagnostics without a prospectively
specified bar.

\subsubsection{Microgrid gates}

\paragraph{Design and calibration.}
D1 is the design-time discrimination gate. For every evaluation site and
seed $0$--$5$, the constant-reserve oracle uses $200{,}000$ CRN samples
on a 41-point reserve grid with $\alpha=0.1$, $H=24$, a flat tariff, and
$p_{\mathrm{cap}}$ at $q_{0.95}$. It computes
$r_{\mathrm{cvar}}-r_{\mathrm{mean}}$ and is met only if every seed's median
over evaluation sites is at least $0.15$. One adjustment to $q_{0.90}$ was
permitted before stopping. The gate was met without that adjustment: the
six per-seed medians were $0.1625$, $0.175$, $0.175$, $0.175$, $0.1625$,
and $0.175$, and the per-seed minimum over sites was $0.15$. The planning
evidence is therefore conditional on this passed screen.

$G_{E2}$ is a pooled cross-client calibration gate. It requires median
one-step log-load skill over persistence above zero, every site's skill
above $-0.05$, pooled PIT coverage within $0.05$ of nominal at the five
levels $0.05$, $0.25$, $0.5$, $0.75$, and $0.95$, and fewer than $0.01$ of
pooled held-out innovations outside the fitted support. Per-site PIT
deviations are diagnostics rather than gate components.

\paragraph{Planning and ensemble comparisons.}
V1 tests the paired episode headroom-loss rate against mean rollout with a
one-sided upper bound below zero. B1 requires the one-sided $95\%$ bootstrap
lower bound on the relative $\CVaR_{0.1}$ gain to be at least $0.10$. B2
uses the first-action decision-variance ratio and the same $1/3$ upper bar;
B2eff applies that bar to planned-reserve variance as an exploratory check.
Both budget-match gates require the two $N_{\mathrm{eval}}$ values to be
within $(K-1)H+10^{-6}=96.000001$, and both hold at every site.

The trajectory-count gate varies only $S$: no $S\leq200$ may enter the
$5\%$ band, and the one-sided lower bound at $S=200$ must exceed $0.05$.
The coupled-sampler gate instead asks whether some $S\leq50$ is equivalent
within the same $5\%$ band. $G_{E3}$ checks mixture PIT deviation and
member-level skill and support for the 25 bootstrap-resampled models. The
replay-scope design check applies the D1 rule to the execution-scope choice.
E1 and E1v compare the quantile ensemble with PETS on tail cost and V1 using
the B1 bound. E2, E3, and E4 are recorded diagnostics for epistemic
decomposition, a material premium with ratio at least $0.01$, and
shift-amplified premium. The BDG2 check repeats D1 with the same bar,
permitted adjustment, and stopping rule; it is not met at either cap
(Section~\ref{sec:negatives}).

\subsubsection{Baseline disclosures and claim scope}

The $\beta$ sweep shares the headline's
repeats, and $\beta = 0$ is the Gaussian row, so on overload every swept
$\beta$ is weaker than Gaussian and the strongest recorded Monte Carlo
baseline there is the Gaussian control with the iCEM loop. Because
selection and evaluation reuse repeats, inference involving the selected
$\beta$ is descriptive rather than confirmatory. The iCEM-style
row highlights three departures from the published iCEM algorithm: carried
elites replace fresh draws rather than being appended (budget neutrality),
the coloured noise is normalized to unit per-position variance, and
carried elites are re-scored each iteration, which is information-free for
a deterministic evaluator. The record documents eight baseline deviations
across the complete iCEM-style and \dsCEMstyle{} sampler implementation.

Taken together, the gates support the comparisons with noise-blind planning
that are stated in the main text. They do not establish uniform superiority
over budget-matched Monte Carlo, a benefit from ensemble allocation, or
transfer through the CityLearn surrogate. The claims in the paper retain
those limits.

\subsection{Setup and measurement detail displaced from the main text}
\label{sec:appendix-setup-detail}

This section collects checks that support the experimental setup but are not
needed to read the main results. The material is grouped into search settings,
record-level robustness checks, theory diagnostics, and limitations.

\subsubsection{Search settings}

\paragraph{Elite fraction (Section~\ref{sec:setup}).} The elite fraction
differs by evaluator: $0.156$ for QPP and Monte Carlo, $0.031$ for mean
rollout, $0.093$ for moment matching.
At mean rollout's fraction QPP's tail changes by $-0.3824$ on overload and
$-18.9734$ on inventory, under a tenth of either gap, so no B1 or C2
verdict depends on this confound.

\paragraph{Sampler colour.} $\beta$ is chosen
per environment by a prospectively specified baseline-favouring rule (lowest
realized $\CVaR_{0.1}$ of that sampler's Monte Carlo cell over $\{1, 2,
3\}$): $1$ for both samplers on overload; $2$ (iCEM-style) and $1$
(\dsCEMstyle{}) on inventory.

\subsubsection{Numerical validation}

An earlier evaluator included a pool-adjacent-violators projection that could
alter negatively loaded inventory trajectories. All affected inventory
results reported in this paper were recomputed without this projection. For
the controlled overload and synthetic studies, recomputation reproduced the
reported outputs exactly. In the remaining overload and surrogate studies,
the measured projection displacement is zero on the evaluated trajectories,
so removing the operation leaves the reported outputs unchanged. Repeating
Gate 1b without the projection gives the same error, $0.0020$.

The displayed overload table uses a $10{,}000$-replicate bootstrap interval
and gives a B1 lower bound of $-0.8275$. An additional sensitivity analysis
with $200{,}000$ replicates gives a B1 overload bound of $-0.7739$
(SE $0.0106$), a B1 inventory bound of $0.2237$ (SE $0.0005$), and G2 of
$0.0124$ (SE $0.0001$). No criterion outcome changes. Because these analyses
use different resampling counts, their numerical estimates are reported
separately.

\subsubsection{Theory diagnostics}

\paragraph{The spread law, measured (Section~\ref{sec:theory}).} On the
Gaussian monotone-diagonal map the measured gap between propagated and
independent-process quantiles is $0.0048$ at $H = 1$ (where the couplings
coincide and the value is the reference's own error), then $0.6737$,
$1.0597$, $1.4694$ at $H = 3, 5, 10$; the $H = 10$ value lies inside
$[\Rspread(10, 0.5) - 1, \Rspread(10, 0.9) - 1] = [0.7304, 2.0291]$, and
simulation matches~\eqref{eq:spread} within $2\%$ on $a \in \{0.5, 0.7,
0.9\}$, $H \le 10$.

\paragraph{Level-curve interpolation error (Section~\ref{sec:theory}).}
Relative error $0.1344$ / $0.1100$ / $0.0978$ at $K = 3$ / $5$ / $9$ on the
monotone map.

\begin{figure}[h]
\centering
\includegraphics[width=0.97\textwidth]{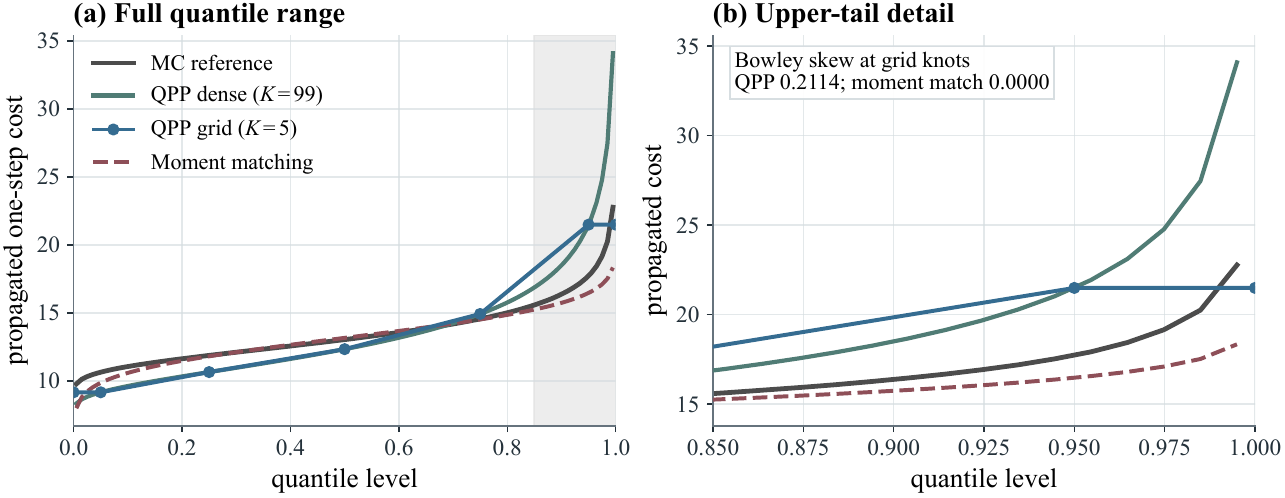}
\caption{One-step cost quantiles under standardized lognormal noise. The
deployed five-level QPP approximation is shown together with a dense-grid
reference; both retain positive skew and upper-tail curvature, whereas
Gaussian moment matching is symmetric by construction. This diagnostic
concerns distributional shape, not calibration or planning performance.}
\label{fig:skew-mechanism}
\end{figure}

\FloatBarrier
\subsubsection{Limitations and transfer checks}

\paragraph{KPI simulator (Section~\ref{sec:negatives}).} Skill $0.0122$;
PIT deviation $0.1378$; $3.39\%$ of residuals out of support.

\paragraph{Measurement limitations (Section~\ref{sec:negatives}).} The
five-level grid has only the $0.75$--$0.95$ segment and a constant extension
above $0.95$ to describe the upper $20\%$ planning tail. Costs are charged
on the pre-step state, while violations are measured after transition, so
the terminal state is unpriced and the last action can affect a violation
without incurring the corresponding state cost. Realized tail metrics mix
optimizer and execution randomness across repeats. Runtime and memory were
not matched.

\paragraph{Second-dataset diagnosis (Section~\ref{sec:negatives}).} The
unchanged microgrid construction has $113$ eligible buildings on
BDG2~\citep{miller2020bdg2}; sixteen span four campuses and four are held
out. D1 fails at both permitted caps (median gap $0.0625$--$0.075$ against
$0.15$), stopping the study before calibration or planning. Two buildings
have zero gap for opposite reasons: one anomaly scale makes overload
dominate every reserve, while the other is too mild for reserve to repay
degradation. ECL's cohort spans anomaly scale $0.1166$--$0.7962$;
persistence is $0.9441$ here against $0.7525$ on ECL, and the coupled-to-
executed comparison is approximately $1.03\times10^9$ versus $45.7$, a
factor of $2.2\times10^7$ against ECL's factor of about thirty-four. This
indicates severe surrogate looseness under the present construction. The two
datasets suggest that persistence and anomaly scale may
govern comonotonic over-dispersion, but they do not establish a cutoff or a
general population scope.

\subsection{Proofs}
\label{sec:appendix-proofs}

\paragraph{The two order-restriction operators.} Pool-adjacent violators
is the Euclidean projection onto the monotone cone: it keeps each entry's
position and replaces violating runs by their means, and it is
non-expansive because it is the metric projection onto a closed convex
cone. Rearrangement sorts the values: it keeps the multiset and permutes
it, and it is non-expansive by the Hardy--Littlewood--P\'olya rearrangement
inequality, not by convexity. They disagree in general ($(3, 1, 2) \mapsto
(2, 2, 2)$ versus $(1, 2, 3)$), so a lemma about one does not transfer to
the other.

\begin{proof}
\textbf{of Proposition~\ref{prop:exact}, Part (a).}
Induction on $t$: $z_0^{(k)} = z_0 = Z_0(q_k)$, and if $z_t^{(k)} =
Z_t(q_k)$ then~\eqref{eq:propagate} is the definition of $Z_{t+1}(q_k)$
with $F^{-1}(U)$ evaluated at $U = q_k$. Hence $C_t^{(k)} = c(Z_t(q_k),
u_t)$, and the level curve interpolates $u \mapsto c(Z_t(u), u_t)$ at the
knots.

\textbf{Part (b).}
On every compact subinterval, continuity and vanishing interior
mesh give uniform interpolation convergence. Endpoint knots tending to
$0$ and $1$, together with uniform integrability, then give $L^1$
convergence on $(0,1)$. Hence the mean converges, and
$|\CVaR_\alpha(X)-\CVaR_\alpha(Y)|\leq
\alpha^{-1}\mathbb{E}|X-Y|$ gives CVaR convergence under the common-$U$
coupling.

\textbf{Part (c).}
For the marginal-quantile statement, if
$Z_{t,i}(u)$ is non-decreasing in $u$, then
\[
\Pr[Z_{t,i}(U)<Z_{t,i}(q)]\le q
\quad\text{and}\quad
\Pr[Z_{t,i}(U)\le Z_{t,i}(q)]\ge q.
\]
Thus $Z_{t,i}(q)$ is a generalized $q$-quantile (with equality in the
second relation when the map is strictly increasing). Non-decreasing
coordinates propagate: if every coordinate of
$Z_t(u)$ is non-decreasing in $u$, then $\frac{d}{du} Z_{t+1}(u) = \big(J_f
+ F^{-1}(u) J_\sigma\big)\, \frac{d}{du} Z_t(u) + \sigma\, \frac{d}{du}
F^{-1}(u)$ has non-negative entries whenever the bracketed matrix and
$\sigma$ do, since $F^{-1}$ is non-decreasing. The base derivative is
$dZ_0/du=0$, and $dZ_1/du=\sigma(z_0,u_0)(F^{-1})'(u)\geq0$ almost
everywhere; local absolute continuity
then preserves monotonicity. The condition is required for every $u \in (0, 1)$ (non-negativity at the
$K$ propagated levels alone says nothing about the order between knots);
the diagnostic in Appendix~\ref{sec:appendix-setup-detail} checks only that the $K$
propagated trajectories are ordered along $K$, a necessary condition.
\end{proof}

\begin{proof}
\textbf{of Remark~\ref{rem:spread}.}
From $z_0 = 0$, the coupled process gives
\[
  Z_H=s\,\varepsilon\sum_{t<H}a^t
  =s\,\varepsilon\frac{1-a^H}{1-a},
  \qquad
  \mathrm{sd}(Z_H)=s\frac{1-a^H}{1-a}.
\]
For the independent process,
\[
  \widetilde Z_H=s\sum_{t<H}a^t\varepsilon_{H-1-t},
  \qquad
  \mathrm{sd}(\widetilde Z_H)
  =s\sqrt{\frac{1-a^{2H}}{1-a^2}}.
\]
Divide the two expressions and use
$1-a^{2H}=(1-a^H)(1+a^H)$ and $1-a^2=(1-a)(1+a)$.
\end{proof}

\begin{proof}
\textbf{of Proposition~\ref{prop:conservative}.}
The argument has four steps: write the total cost as a function of the shock
sequence, establish its monotonicity and supermodularity, apply the
comonotonic upper bound, and then use the stop-loss representation of CVaR.
Under the hypotheses each per-step cost is $g_t(v_t^\top Z_t + \beta_t)$
with $v_t \ge 0$ and $g_t$ convex non-decreasing, and $Z_t$ is a
non-negative affine image of the shocks, so the summed cost is
$\Phi(\varepsilon_1, \dots, \varepsilon_H) = \sum_t g_t\big(w_t^\top
\varepsilon_{1:t} + b_t\big)$ with $w_t \ge 0$. (A general convex
coordinatewise non-decreasing cost need not be supermodular, e.g.
$\max(x_1, x_2)$; the ridge form is what the proof uses.) Each summand is
non-decreasing and supermodular in $\varepsilon$ (for twice-differentiable
$g_t$, $\partial^2 / \partial \varepsilon_i \partial \varepsilon_j = g_t''
w_{t,i} w_{t,j} \ge 0$; the general case follows by approximation), and
sums preserve both properties. For any non-decreasing convex $\psi$, $\psi
\circ \Phi$ is again non-decreasing and supermodular (the cross derivative
is $\psi'' \,\partial_i \Phi\, \partial_j \Phi + \psi' \,\partial_{ij}
\Phi \ge 0$). The comonotonic coupling maximizes $\mathbb{E}[h(\varepsilon)]$
over all couplings with the same marginals for every supermodular $h$
(the comonotonic upper bound, \citealp{kaasdhaenegoovaerts2000comonotonic}),
so $\mathbb{E}[\psi(\Phi^{\mathrm{com}})] \ge \mathbb{E}[\psi(\Phi^{\mathrm{ind}})]$
for every non-decreasing convex $\psi$, i.e.\ $\Phi^{\mathrm{com}}$
dominates $\Phi^{\mathrm{ind}}$ in the increasing convex order. For every
$v$, the stop-loss function $x\mapsto(x-v)_+$ is increasing and convex,
so
\[
\mathbb{E}[(\Phi^{\mathrm{ind}}-v)_+]
\leq \mathbb{E}[(\Phi^{\mathrm{com}}-v)_+] .
\]
Applying this inequality in
$\CVaR_\alpha(X)=\inf_v\{v+\alpha^{-1}\mathbb{E}[(X-v)_+]\}$ and then
taking the infimum over $v$ gives the stated CVaR ordering. Here $\alpha$
is the upper-tail mass, matching Section~\ref{sec:method}. The
counterexample of Section~\ref{sec:theory} violates the
hypotheses twice (a rotation mixes signs in $w_t$; a concave cost breaks
supermodularity) and reverses the inequality.
\end{proof}

\subsection{Main-text tables in full, and tables moved for space}
\label{sec:appendix-tables-full}

The tables in this section expand results summarized in the main text. They
are ordered as comparisons with mean rollout, the budget-matched Monte Carlo
comparison, within-sampler composition results, the $\lam$ sweep, and the
interaction contrasts. Section~\ref{sec:appendix-composition} begins only
after this complete-table block.
Tables write a two-sided interval as $m_{-a}^{+b}$, where $m$ is the estimate
and $a,b$ are the lower and upper $95\%$ offsets. One-sided intervals use
$\infty$ on the open side; a lone superscript gives an upper offset.

\begin{sidewaystable}[p]
\centering
\footnotesize
\setlength{\tabcolsep}{3pt}
\begin{tabular}{@{}>{\raggedright\arraybackslash}p{0.30\linewidth} >{\raggedright\arraybackslash}p{0.10\linewidth} >{\raggedright\arraybackslash}p{0.07\linewidth} >{\centering\arraybackslash}p{0.15\linewidth} >{\raggedleft\arraybackslash}p{0.10\linewidth} >{\raggedleft\arraybackslash}p{0.10\linewidth} r r@{}}
\toprule
Comparison & Metric & Alt. & Estimate with 95\% CI & $p$ & $p_\mathrm{adj}$ & $n$ & $n_\mathrm{boot}$ \\
\midrule
QPP lambda=0 vs mean & mean & lower & \asymci{1.762}{\infty}{0.412} & 1 & 1 & 200 & 10000 \\
QPP lambda=0 vs mean & CVaR & lower & \asymci{-5.091}{\infty}{4.008} & 0.007199 & 0.09359 & 200 & 10000 \\
QPP lambda=0 vs mean & violation & lower & \asymci{-0.025}{\infty}{0.015} & 0.03125 & 0.3125 & 200 & 10000 \\
QPP lambda=0.5 vs mean & mean & lower & \asymci{1.961}{\infty}{0.409} & 1 & 1 & 200 & 10000 \\
QPP lambda=0.5 vs mean & CVaR & lower & \asymci{-4.826}{\infty}{3.9985} & 0.0128 & 0.1536 & 200 & 10000 \\
QPP lambda=0.5 vs mean & violation & lower & \asymci{-0.02}{\infty}{0.015} & 0.0625 & 0.5625 & 200 & 10000 \\
QPP lambda=1 vs mean & mean & lower & \asymci{1.994}{\infty}{0.408} & 1 & 1 & 200 & 10000 \\
QPP lambda=1 vs mean & CVaR & lower & \asymci{-4.796}{\infty}{4.0167} & 0.0139 & 0.1536 & 200 & 10000 \\
QPP lambda=1 vs mean & violation & lower & \asymci{-0.02}{\infty}{0.015} & 0.0625 & 0.5625 & 200 & 10000 \\
moment match vs mean & mean & two-sided & \asymci{0.00242}{0.009161}{0.007514} & 0.602 & 1 & 200 & 10000 \\
moment match vs mean & CVaR & two-sided & \asymci{-0.1287}{0.0755}{0.06637} & 0.0002 & 0.003 & 200 & 10000 \\
moment match vs mean & violation & two-sided & \asymci{0}{0}{0} & 1 & 1 & 200 & 10000 \\
MC (S=5) vs mean & mean & two-sided & \asymci{0.1114}{0.16426}{0.1113} & 0.09769 & 0.6838 & 200 & 10000 \\
MC (S=5) vs mean & CVaR & two-sided & \asymci{-1.616}{1.480}{1.0522} & 0.0002 & 0.003 & 200 & 10000 \\
MC (S=5) vs mean & violation & two-sided & \asymci{0}{0}{0} & 1 & 1 & 200 & 10000 \\
\bottomrule
\end{tabular}
 \caption{Overload environment, every row of the comparison against mean
rollout; Section~\ref{sec:results-mean} cites the $\lam = 0.5$ row. One Holm
family.}
\label{tab:overload-mean-full}
\end{sidewaystable}

\begin{sidewaystable}[p]
\centering
\footnotesize
\setlength{\tabcolsep}{3pt}
\begin{tabular}{@{}>{\raggedright\arraybackslash}p{0.30\linewidth} >{\raggedright\arraybackslash}p{0.10\linewidth} >{\raggedright\arraybackslash}p{0.07\linewidth} >{\centering\arraybackslash}p{0.15\linewidth} >{\raggedleft\arraybackslash}p{0.10\linewidth} >{\raggedleft\arraybackslash}p{0.10\linewidth} r r@{}}
\toprule
Comparison & Metric & Alt. & Estimate with 95\% CI & $p$ & $p_\mathrm{adj}$ & $n$ & $n_\mathrm{boot}$ \\
\midrule
QPP lambda=0 vs mean & mean & lower & \asymci{59.68}{\infty}{11.33} & 1 & 1 & 200 & 10000 \\
QPP lambda=0 vs mean & CVaR & lower & \asymci{-125}{\infty}{59.3} & $9.999\times10^{-5}$ & 0.0009999 & 200 & 10000 \\
QPP lambda=0 vs mean & violation & two-sided & \asymci{-0.66}{0.065}{0.065} & $3.673\times10^{-40}$ & $4.775\times10^{-39}$ & 200 & 10000 \\
QPP lambda=0.5 vs mean & mean & lower & \asymci{78.24}{\infty}{11.12} & 1 & 1 & 200 & 10000 \\
QPP lambda=0.5 vs mean & CVaR & lower & \asymci{-106.9}{\infty}{59.84} & 0.0002 & 0.0014 & 200 & 10000 \\
QPP lambda=0.5 vs mean & violation & two-sided & \asymci{-0.67}{0.065}{0.065} & $9.184\times10^{-41}$ & $1.378\times10^{-39}$ & 200 & 10000 \\
QPP lambda=1 vs mean & mean & lower & \asymci{78.82}{\infty}{11.33} & 1 & 1 & 200 & 10000 \\
QPP lambda=1 vs mean & CVaR & lower & \asymci{-106.2}{\infty}{59.53} & 0.0002 & 0.0014 & 200 & 10000 \\
QPP lambda=1 vs mean & violation & two-sided & \asymci{-0.67}{0.065}{0.065} & $9.184\times10^{-41}$ & $1.378\times10^{-39}$ & 200 & 10000 \\
moment match vs mean & mean & two-sided & \asymci{-10.06}{2.39}{2.276} & $9.999\times10^{-5}$ & 0.0009999 & 200 & 10000 \\
moment match vs mean & CVaR & two-sided & \asymci{-39.25}{6.31}{5.96} & 0.0002 & 0.0014 & 200 & 10000 \\
moment match vs mean & violation & two-sided & \asymci{-0.1}{0.045}{0.04} & $1.097\times10^{-5}$ & 0.0001206 & 200 & 10000 \\
MC (S=5) vs mean & mean & two-sided & \asymci{-23.47}{7.07}{6.82} & $9.999\times10^{-5}$ & 0.0009999 & 200 & 10000 \\
MC (S=5) vs mean & CVaR & two-sided & \asymci{-124.8}{12.3}{19.2} & 0.0002 & 0.0014 & 200 & 10000 \\
MC (S=5) vs mean & violation & two-sided & \asymci{-0.335}{0.065}{0.065} & $1.355\times10^{-20}$ & $1.626\times10^{-19}$ & 200 & 10000 \\
\bottomrule
\end{tabular}
 \caption{Inventory environment, every row of the comparison against mean
rollout; Section~\ref{sec:results-mean} cites the $\lam = 0.5$ row. One Holm
family.}
\label{tab:inventory-mean-full}
\end{sidewaystable}

\begin{sidewaystable}[p]
\centering
\footnotesize
\setlength{\tabcolsep}{3pt}
\begin{tabular}{@{}>{\raggedright\arraybackslash}p{0.30\linewidth} >{\raggedright\arraybackslash}p{0.10\linewidth} >{\raggedright\arraybackslash}p{0.07\linewidth} >{\centering\arraybackslash}p{0.15\linewidth} >{\raggedleft\arraybackslash}p{0.10\linewidth} >{\raggedleft\arraybackslash}p{0.10\linewidth} r r@{}}
\toprule
Comparison & Metric & Alt. & Estimate with 95\% CI & $p$ & $p_\mathrm{adj}$ & $n$ & $n_\mathrm{boot}$ \\
\midrule
QPP lambda=0 vs MC (S=5) & mean & two-sided & \asymci{1.651}{0.487}{0.355} & $9.999\times10^{-5}$ & 0.0015 & 200 & 10000 \\
QPP lambda=0 vs MC (S=5) & CVaR & two-sided & \asymci{-3.474}{4.708}{3.47839} & 0.05099 & 0.5099 & 200 & 10000 \\
QPP lambda=0 vs MC (S=5) & violation & two-sided & \asymci{-0.025}{0.025}{0.020} & 0.0625 & 0.5625 & 200 & 10000 \\
QPP lambda=0.5 vs MC (S=5) & mean & two-sided & \asymci{1.849}{0.486}{0.355} & $9.999\times10^{-5}$ & 0.0015 & 200 & 10000 \\
QPP lambda=0.5 vs MC (S=5) & CVaR & two-sided & \asymci{-3.21}{4.531}{3.4608} & 0.08059 & 0.6183 & 200 & 10000 \\
QPP lambda=0.5 vs MC (S=5) & violation & two-sided & \asymci{-0.02}{0.02}{0.015} & 0.125 & 0.6183 & 200 & 10000 \\
QPP lambda=1 vs MC (S=5) & mean & two-sided & \asymci{1.882}{0.476}{0.360} & $9.999\times10^{-5}$ & 0.0015 & 200 & 10000 \\
QPP lambda=1 vs MC (S=5) & CVaR & two-sided & \asymci{-3.179}{4.484}{3.4674} & 0.08819 & 0.6183 & 200 & 10000 \\
QPP lambda=1 vs MC (S=5) & violation & two-sided & \asymci{-0.02}{0.02}{0.015} & 0.125 & 0.6183 & 200 & 10000 \\
mean vs MC (S=5) & mean & two-sided & \asymci{-0.1114}{0.1132}{0.16136} & 0.09829 & 0.6183 & 200 & 10000 \\
mean vs MC (S=5) & CVaR & two-sided & \asymci{1.616}{1.0424}{1.461} & 0.0002 & 0.0024 & 200 & 10000 \\
mean vs MC (S=5) & violation & two-sided & \asymci{0}{0}{0} & 1 & 1 & 200 & 10000 \\
moment match vs MC (S=5) & mean & two-sided & \asymci{-0.109}{0.1044}{0.15509} & 0.07729 & 0.6183 & 200 & 10000 \\
moment match vs MC (S=5) & CVaR & two-sided & \asymci{1.488}{0.9868}{1.408} & 0.0002 & 0.0024 & 200 & 10000 \\
moment match vs MC (S=5) & violation & two-sided & \asymci{0}{0}{0} & 1 & 1 & 200 & 10000 \\
\bottomrule
\end{tabular}
 \caption{Overload environment: every planner against $S = 5$ Monte Carlo at
matched $\Nevals$, two-sided. One Holm family.}
\label{tab:overload-mc}
\end{sidewaystable}

\begin{sidewaystable}[p]
\centering
\footnotesize
\setlength{\tabcolsep}{3pt}
\begin{tabular}{@{}>{\raggedright\arraybackslash}p{0.30\linewidth} >{\raggedright\arraybackslash}p{0.10\linewidth} >{\raggedright\arraybackslash}p{0.07\linewidth} >{\centering\arraybackslash}p{0.15\linewidth} >{\raggedleft\arraybackslash}p{0.10\linewidth} >{\raggedleft\arraybackslash}p{0.10\linewidth} r r@{}}
\toprule
Comparison & Metric & Alt. & Estimate with 95\% CI & $p$ & $p_\mathrm{adj}$ & $n$ & $n_\mathrm{boot}$ \\
\midrule
Gaussian candidates with QPP vs Gaussian candidates with MC & CVaR & lower & \asymci{-3.132}{\infty}{3.138725} & 0.05119 & 0.3367 & 200 & 10000 \\
Gaussian candidates with QPP vs Gaussian candidates with mean rollout & CVaR & lower & \asymci{-4.87}{\infty}{4.0134} & 0.0147 & 0.2058 & 200 & 10000 \\
iCEM-style candidates with QPP vs iCEM-style candidates with MC & CVaR & lower & \asymci{-3.392}{\infty}{3.389419} & 0.05 & 0.3367 & 200 & 10000 \\
iCEM-style candidates with QPP vs iCEM-style candidates with mean rollout & CVaR & lower & \asymci{-4.702}{\infty}{4.0436} & 0.0187 & 0.2067 & 200 & 10000 \\
dsCEM-style candidates with QPP vs dsCEM-style candidates with MC & CVaR & lower & \asymci{-3.426}{\infty}{3.51262} & 0.05739 & 0.3367 & 200 & 10000 \\
dsCEM-style candidates with QPP vs dsCEM-style candidates with mean rollout & CVaR & lower & \asymci{-4.7}{\infty}{4.0909} & 0.0194 & 0.2067 & 200 & 10000 \\
Gaussian candidates, iCEM loop with QPP vs Gaussian candidates, iCEM loop with MC & CVaR & lower & \asymci{-2.368}{\infty}{2.7465} & 0.09619 & 0.3367 & 200 & 10000 \\
Gaussian candidates, iCEM loop with QPP vs Gaussian candidates, iCEM loop with mean rollout & CVaR & lower & \asymci{-4.783}{\infty}{4.0475} & 0.0159 & 0.2067 & 200 & 10000 \\
Gaussian candidates, dsCEM loop with QPP vs Gaussian candidates, dsCEM loop with MC & CVaR & lower & \asymci{-2.403}{\infty}{2.6368} & 0.07969 & 0.3367 & 200 & 10000 \\
Gaussian candidates, dsCEM loop with QPP vs Gaussian candidates, dsCEM loop with mean rollout & CVaR & lower & \asymci{-4.768}{\infty}{4.0468} & 0.0172 & 0.2067 & 200 & 10000 \\
Gaussian candidates, carryover only with QPP vs Gaussian candidates, carryover only with MC & CVaR & lower & \asymci{-2.835}{\infty}{2.81064} & 0.0481 & 0.3367 & 200 & 10000 \\
Gaussian candidates, carryover only with QPP vs Gaussian candidates, carryover only with mean rollout & CVaR & lower & \asymci{-4.788}{\infty}{4.0771} & 0.0159 & 0.2067 & 200 & 10000 \\
Gaussian candidates, schedule only with QPP vs Gaussian candidates, schedule only with MC & CVaR & lower & \asymci{-2.054}{\infty}{2.6209} & 0.1468 & 0.3367 & 200 & 10000 \\
Gaussian candidates, schedule only with QPP vs Gaussian candidates, schedule only with mean rollout & CVaR & lower & \asymci{-4.788}{\infty}{4.0626} & 0.0171 & 0.2067 & 200 & 10000 \\
\bottomrule
\end{tabular}
 \caption{Composition study, overload: within each sampler row, QPP against
$S = 5$ Monte Carlo and against mean rollout on absolute $\CVaR_{0.1}$
(one-sided lower alternative). Rows: Gaussian, iCEM-style, \dsCEMstyle{}, and the
four Gaussian loop controls. One Holm family.}
\label{tab:within-overload}
\end{sidewaystable}

\begin{sidewaystable}[p]
\centering
\footnotesize
\setlength{\tabcolsep}{3pt}
\begin{tabular}{@{}>{\raggedright\arraybackslash}p{0.30\linewidth} >{\raggedright\arraybackslash}p{0.10\linewidth} >{\raggedright\arraybackslash}p{0.07\linewidth} >{\centering\arraybackslash}p{0.15\linewidth} >{\raggedleft\arraybackslash}p{0.10\linewidth} >{\raggedleft\arraybackslash}p{0.10\linewidth} r r@{}}
\toprule
Comparison & Metric & Alt. & Estimate with 95\% CI & $p$ & $p_\mathrm{adj}$ & $n$ & $n_\mathrm{boot}$ \\
\midrule
Gaussian candidates with QPP vs Gaussian candidates with MC & relative $\CVaR_{0.1}$ gain & upper & \asymci{-0.1231}{0.5967}{\infty} & 0.7315 & 1 & 200 & 10000 \\
Gaussian candidates with QPP vs Gaussian candidates with mean rollout & relative $\CVaR_{0.1}$ gain & upper & \asymci{0.3958}{0.1718}{\infty} & 0.0002 & 0.0028 & 200 & 10000 \\
iCEM-style candidates with QPP vs iCEM-style candidates with MC & relative $\CVaR_{0.1}$ gain & upper & \asymci{-0.4792}{0.9628}{\infty} & 0.9434 & 1 & 200 & 10000 \\
iCEM-style candidates with QPP vs iCEM-style candidates with mean rollout & relative $\CVaR_{0.1}$ gain & upper & \asymci{0.3508}{0.1844}{\infty} & 0.0011 & 0.008799 & 200 & 10000 \\
dsCEM-style candidates with QPP vs dsCEM-style candidates with MC & relative $\CVaR_{0.1}$ gain & upper & \asymci{-0.137}{0.6009}{\infty} & 0.7491 & 1 & 200 & 10000 \\
dsCEM-style candidates with QPP vs dsCEM-style candidates with mean rollout & relative $\CVaR_{0.1}$ gain & upper & \asymci{0.3624}{0.1814}{\infty} & 0.0008999 & 0.008099 & 200 & 10000 \\
Gaussian candidates, iCEM loop with QPP vs Gaussian candidates, iCEM loop with MC & relative $\CVaR_{0.1}$ gain & upper & \asymci{-0.3928}{1.0072}{\infty} & 0.8945 & 1 & 200 & 10000 \\
Gaussian candidates, iCEM loop with QPP vs Gaussian candidates, iCEM loop with mean rollout & relative $\CVaR_{0.1}$ gain & upper & \asymci{0.3842}{0.1742}{\infty} & 0.0002 & 0.0028 & 200 & 10000 \\
Gaussian candidates, dsCEM loop with QPP vs Gaussian candidates, dsCEM loop with MC & relative $\CVaR_{0.1}$ gain & upper & \asymci{-0.1154}{0.5706}{\infty} & 0.7291 & 1 & 200 & 10000 \\
Gaussian candidates, dsCEM loop with QPP vs Gaussian candidates, dsCEM loop with mean rollout & relative $\CVaR_{0.1}$ gain & upper & \asymci{0.4047}{0.1675}{\infty} & 0.0002 & 0.0028 & 200 & 10000 \\
Gaussian candidates, carryover only with QPP vs Gaussian candidates, carryover only with MC & relative $\CVaR_{0.1}$ gain & upper & \asymci{-0.1415}{0.5717}{\infty} & 0.758 & 1 & 200 & 10000 \\
Gaussian candidates, carryover only with QPP vs Gaussian candidates, carryover only with mean rollout & relative $\CVaR_{0.1}$ gain & upper & \asymci{0.3922}{0.1749}{\infty} & 0.0002 & 0.0028 & 200 & 10000 \\
Gaussian candidates, schedule only with QPP vs Gaussian candidates, schedule only with MC & relative $\CVaR_{0.1}$ gain & upper & \asymci{-0.3992}{0.8368}{\infty} & 0.9313 & 1 & 200 & 10000 \\
Gaussian candidates, schedule only with QPP vs Gaussian candidates, schedule only with mean rollout & relative $\CVaR_{0.1}$ gain & upper & \asymci{0.3898}{0.1735}{\infty} & 0.0002 & 0.0028 & 200 & 10000 \\
\bottomrule
\end{tabular}
 \caption{Composition study, inventory: the same within-row comparisons on
the relative $\CVaR_{0.1}$ gain (one-sided upper alternative). One Holm family.}
\label{tab:within-inventory}
\end{sidewaystable}

\begin{sidewaystable}[p]
\centering
\footnotesize
\setlength{\tabcolsep}{3pt}
\begin{tabular}{@{}>{\raggedright\arraybackslash}p{0.30\linewidth} >{\raggedright\arraybackslash}p{0.10\linewidth} >{\raggedright\arraybackslash}p{0.07\linewidth} >{\centering\arraybackslash}p{0.15\linewidth} >{\raggedleft\arraybackslash}p{0.10\linewidth} >{\raggedleft\arraybackslash}p{0.10\linewidth} r r@{}}
\toprule
Comparison & Metric & Alt. & Estimate with 95\% CI & $p$ & $p_\mathrm{adj}$ & $n$ & $n_\mathrm{boot}$ \\
\midrule
lam=0.25 vs lam=0.00 & mean & two-sided & \asymci{0.001852}{0.0008716}{0.000802} & 0.0002 & 0.0016 & 100 & 10000 \\
lam=0.25 vs lam=0.00 & CVaR & two-sided & \asymci{-0.003421}{0.002766}{0.00344513} & 0.05159 & 0.1032 & 100 & 10000 \\
lam=0.50 vs lam=0.00 & mean & two-sided & \asymci{0.002868}{0.000960}{0.000895} & 0.0002 & 0.0016 & 100 & 10000 \\
lam=0.50 vs lam=0.00 & CVaR & two-sided & \asymci{-0.003208}{0.002866}{0.0038272} & 0.1116 & 0.1116 & 100 & 10000 \\
lam=0.75 vs lam=0.00 & mean & two-sided & \asymci{-0.01655}{0.00157}{0.00151} & 0.0002 & 0.0016 & 100 & 10000 \\
lam=0.75 vs lam=0.00 & CVaR & two-sided & \asymci{-0.02855}{0.00394}{0.00635} & 0.0002 & 0.0016 & 100 & 10000 \\
lam=1.00 vs lam=0.00 & mean & two-sided & \asymci{-0.01669}{0.00157}{0.00150} & 0.0002 & 0.0016 & 100 & 10000 \\
lam=1.00 vs lam=0.00 & CVaR & two-sided & \asymci{-0.02866}{0.00390}{0.00633} & 0.0002 & 0.0016 & 100 & 10000 \\
\bottomrule
\end{tabular}
 \caption{Synthetic environment: realized mean and $\CVaR_{0.1}$ of QPP at
$\lam \in \{0.25, 0.5, 0.75, 1\}$ against $\lam = 0$, two-sided,
$n = 100$. One Holm family.}
\label{tab:lam-sweep}
\end{sidewaystable}

\begin{sidewaystable}[p]
\centering
\footnotesize
\setlength{\tabcolsep}{3pt}
\begin{tabular}{@{}>{\raggedright\arraybackslash}p{0.30\linewidth} >{\raggedright\arraybackslash}p{0.10\linewidth} >{\raggedright\arraybackslash}p{0.07\linewidth} >{\centering\arraybackslash}p{0.15\linewidth} >{\raggedleft\arraybackslash}p{0.10\linewidth} >{\raggedleft\arraybackslash}p{0.10\linewidth} r r@{}}
\toprule
Comparison & Metric & Alt. & Estimate with 95\% CI & $p$ & $p_\mathrm{adj}$ & $n$ & $n_\mathrm{boot}$ \\
\midrule
iCEM-style candidates with QPP minus MC vs Gaussian candidates with QPP minus MC & CVaR & two-sided & \asymci{-0.2609}{0.6341}{0.6304} & 0.4766 & 1 & 200 & 10000 \\
dsCEM-style candidates with QPP minus MC vs Gaussian candidates with QPP minus MC & CVaR & two-sided & \asymci{-0.2948}{1.0102}{1.1001} & 0.6283 & 1 & 200 & 10000 \\
Gaussian candidates, iCEM loop with QPP minus MC vs Gaussian candidates with QPP minus MC & CVaR & two-sided & \asymci{0.7636}{0.69229}{0.7754} & 0.0274 & 0.274 & 200 & 10000 \\
Gaussian candidates, dsCEM loop with QPP minus MC vs Gaussian candidates with QPP minus MC & CVaR & two-sided & \asymci{0.7287}{0.67308}{0.8193} & 0.03 & 0.274 & 200 & 10000 \\
Gaussian candidates, carryover only with QPP minus MC vs Gaussian candidates with QPP minus MC & CVaR & two-sided & \asymci{0.2961}{0.6100}{0.8179} & 0.4236 & 1 & 200 & 10000 \\
Gaussian candidates, schedule only with QPP minus MC vs Gaussian candidates with QPP minus MC & CVaR & two-sided & \asymci{1.077}{0.7264}{0.935} & 0.0009999 & 0.011 & 200 & 10000 \\
\bottomrule
\end{tabular}
 \caption{Composition study, overload: each row's QPP-minus-Monte-Carlo
$\CVaR_{0.1}$ gap minus the Gaussian row's gap (two-sided; zero means the
sampler or loop leaves QPP's advantage unchanged). Main-text rows of the
prospectively specified family; QPP-minus-mean rows in the appendix.}
\label{tab:interaction-overload}
\end{sidewaystable}

\begin{sidewaystable}[p]
\centering
\footnotesize
\setlength{\tabcolsep}{3pt}
\begin{tabular}{@{}>{\raggedright\arraybackslash}p{0.30\linewidth} >{\raggedright\arraybackslash}p{0.10\linewidth} >{\raggedright\arraybackslash}p{0.07\linewidth} >{\centering\arraybackslash}p{0.15\linewidth} >{\raggedleft\arraybackslash}p{0.10\linewidth} >{\raggedleft\arraybackslash}p{0.10\linewidth} r r@{}}
\toprule
Comparison & Metric & Alt. & Estimate with 95\% CI & $p$ & $p_\mathrm{adj}$ & $n$ & $n_\mathrm{boot}$ \\
\midrule
iCEM-style candidates with QPP minus MC vs Gaussian candidates with QPP minus MC & CVaR & two-sided & \asymci{37.95}{11.59}{9.04} & 0.0002 & 0.0024 & 200 & 10000 \\
dsCEM-style candidates with QPP minus MC vs Gaussian candidates with QPP minus MC & CVaR & two-sided & \asymci{2.485}{11.743}{11.285} & 0.6637 & 1 & 200 & 10000 \\
Gaussian candidates, iCEM loop with QPP minus MC vs Gaussian candidates with QPP minus MC & CVaR & two-sided & \asymci{28.36}{15.13}{13.00} & 0.0005999 & 0.0048 & 200 & 10000 \\
Gaussian candidates, dsCEM loop with QPP minus MC vs Gaussian candidates with QPP minus MC & CVaR & two-sided & \asymci{-1.363}{7.915}{8.151} & 0.7321 & 1 & 200 & 10000 \\
Gaussian candidates, carryover only with QPP minus MC vs Gaussian candidates with QPP minus MC & CVaR & two-sided & \asymci{2.426}{9.821}{10.254} & 0.6643 & 1 & 200 & 10000 \\
Gaussian candidates, schedule only with QPP minus MC vs Gaussian candidates with QPP minus MC & CVaR & two-sided & \asymci{28.62}{13.23}{11.55} & 0.0002 & 0.0024 & 200 & 10000 \\
\bottomrule
\end{tabular}
 \caption{Composition study, inventory: interaction contrasts as in
Table~\ref{tab:interaction-overload}.}
\label{tab:interaction-inventory}
\end{sidewaystable}

\FloatBarrier
\subsection{Composition study: appendix tables}
\label{sec:appendix-composition}

Across the four prospectively gated sampler rows in each of two
environments, QPP is below mean rollout on tail cost in $8/8$ comparisons
(C2) and is not below Monte Carlo in $8/8$ (C1). C3 is satisfied or
inconclusive on overload and not satisfied or inconclusive on inventory. On both closed-form
environments, $\lam=0$ has lower realized mean and $\CVaR_{0.1}$ point
estimates than $\lam=0.5$; because no direct paired family was specified,
the evidence supports the evaluator but not a nonzero-$\lam$ benefit. On the
symmetric control, $\lam$ moves the decision only above $0.5$ and by little
(Table~\ref{tab:lam-sweep}).

\paragraph{Candidates versus loop.} The sampler-only contrasts (each
strong row minus its own loop control, so that only the candidate
distribution differs; Tables~\ref{tab:interaction-overload}
and~\ref{tab:interaction-inventory}) are null on inventory for both
samplers ($9.59$ and $3.85$, intervals covering zero) and on overload for
\dsCEMstyle{}; the coloured iCEM-style candidates move the overload
comparison in QPP's favour ($-1.02$, $p_{\mathrm{adj}} = 0.0056$). The
iCEM loop moves the comparison in Monte Carlo's favour on both
environments (Gaussian control with the iCEM loop: $28.4$ on inventory,
$p_{\mathrm{adj}} = 0.0048$; $0.764$ on overload), the \dsCEMstyle{} loop
on neither. Table~\ref{tab:loop-decomposition} attributes the loop effect
to the decaying population schedule on both environments: the
schedule-only control reproduces the iCEM-loop interaction ($28.6$ and
$1.08$, both Holm-significant; iCEM loop minus schedule alone null on
both), carryover alone (in its budget-neutral, re-scored form) does not
(null on both; iCEM loop minus carryover $25.9$ and $0.468$, both
significant), and the additivity residual is not rejected under the
headline decision rule. This is a measured attribution, not a mechanism
proof: under the best-of-final-elites rule the direct loop-minus-carryover
contrasts lose significance and the overload additivity residual is
rejected ($0.799$, $p_{\mathrm{adj}} = 0.0174$); only the schedule-only
interaction rows keep theirs. (The recorded elite fraction under the
schedule rises from $0.07$ to $0.5$ across iterations, the candidate
explanation.)

\paragraph{Per environment.} On overload the strong samplers reproduce the
Gaussian tail gap within the equivalence margin (C3 satisfied); the C1 gates
are not satisfied, the iCEM-style one by a deciding bound of $0.0076$ against $0$
with standard error $0.0071$ (a boundary outcome read in neither
direction), and QPP pays $1.80$ to $1.99$ of realized mean against both
baselines in every row (Table~\ref{tab:within-mean-overload}): the
overload tail claim and this penalty are one statement. On inventory there
is no QPP-versus-Monte-Carlo win to test (C1 not satisfied in all four gated
rows; mean penalty $98.4$ to $113$ against Monte Carlo in every row), the
iCEM-style loop widens the gap further (interaction $37.95$), and C3 is
not satisfied for the two iCEM-loop rows and inconclusive for the two
\dsCEMstyle{} rows.

\begin{sidewaystable}[p]
\centering
\footnotesize
\setlength{\tabcolsep}{3pt}
\begin{tabular}{@{}>{\raggedright\arraybackslash}p{0.30\linewidth} >{\raggedright\arraybackslash}p{0.10\linewidth} >{\raggedright\arraybackslash}p{0.07\linewidth} >{\centering\arraybackslash}p{0.15\linewidth} >{\raggedleft\arraybackslash}p{0.10\linewidth} >{\raggedleft\arraybackslash}p{0.10\linewidth} r r@{}}
\toprule
Comparison & Metric & Alt. & Estimate with 95\% CI & $p$ & $p_\mathrm{adj}$ & $n$ & $n_\mathrm{boot}$ \\
\midrule
Gaussian candidates with QPP vs Gaussian candidates with MC & mean & two-sided & \asymci{1.812}{0.509}{0.360} & 0.0002 & 0.0028 & 200 & 10000 \\
Gaussian candidates with QPP vs Gaussian candidates with mean rollout & mean & two-sided & \asymci{1.965}{0.643}{0.464} & 0.0002 & 0.0028 & 200 & 10000 \\
iCEM-style candidates with QPP vs iCEM-style candidates with MC & mean & two-sided & \asymci{1.945}{0.548}{0.384} & 0.0002 & 0.0028 & 200 & 10000 \\
iCEM-style candidates with QPP vs iCEM-style candidates with mean rollout & mean & two-sided & \asymci{1.987}{0.627}{0.461} & 0.0002 & 0.0028 & 200 & 10000 \\
dsCEM-style candidates with QPP vs dsCEM-style candidates with MC & mean & two-sided & \asymci{1.95}{0.568}{0.401} & 0.0002 & 0.0028 & 200 & 10000 \\
dsCEM-style candidates with QPP vs dsCEM-style candidates with mean rollout & mean & two-sided & \asymci{1.985}{0.651}{0.465} & 0.0002 & 0.0028 & 200 & 10000 \\
Gaussian candidates, iCEM loop with QPP vs Gaussian candidates, iCEM loop with MC & mean & two-sided & \asymci{1.837}{0.446}{0.314} & 0.0002 & 0.0028 & 200 & 10000 \\
Gaussian candidates, iCEM loop with QPP vs Gaussian candidates, iCEM loop with mean rollout & mean & two-sided & \asymci{1.967}{0.635}{0.464} & 0.0002 & 0.0028 & 200 & 10000 \\
Gaussian candidates, dsCEM loop with QPP vs Gaussian candidates, dsCEM loop with MC & mean & two-sided & \asymci{1.799}{0.454}{0.304} & 0.0002 & 0.0028 & 200 & 10000 \\
Gaussian candidates, dsCEM loop with QPP vs Gaussian candidates, dsCEM loop with mean rollout & mean & two-sided & \asymci{1.972}{0.643}{0.471} & 0.0002 & 0.0028 & 200 & 10000 \\
Gaussian candidates, carryover only with QPP vs Gaussian candidates, carryover only with MC & mean & two-sided & \asymci{1.869}{0.448}{0.331} & 0.0002 & 0.0028 & 200 & 10000 \\
Gaussian candidates, carryover only with QPP vs Gaussian candidates, carryover only with mean rollout & mean & two-sided & \asymci{1.965}{0.621}{0.469} & 0.0002 & 0.0028 & 200 & 10000 \\
Gaussian candidates, schedule only with QPP vs Gaussian candidates, schedule only with MC & mean & two-sided & \asymci{1.838}{0.453}{0.293} & 0.0002 & 0.0028 & 200 & 10000 \\
Gaussian candidates, schedule only with QPP vs Gaussian candidates, schedule only with mean rollout & mean & two-sided & \asymci{1.979}{0.644}{0.471} & 0.0002 & 0.0028 & 200 & 10000 \\
\bottomrule
\end{tabular}
 \caption{Composition study, overload: within each sampler row, QPP against
$S = 5$ Monte Carlo and against mean rollout on the realized mean
(two-sided). QPP pays a mean penalty in every row; this table is the
companion of Table~\ref{tab:within-overload}. One Holm family.}
\label{tab:within-mean-overload}
\end{sidewaystable}

\begin{sidewaystable}[p]
\centering
\footnotesize
\setlength{\tabcolsep}{3pt}
\begin{tabular}{@{}>{\raggedright\arraybackslash}p{0.30\linewidth} >{\raggedright\arraybackslash}p{0.10\linewidth} >{\raggedright\arraybackslash}p{0.07\linewidth} >{\centering\arraybackslash}p{0.15\linewidth} >{\raggedleft\arraybackslash}p{0.10\linewidth} >{\raggedleft\arraybackslash}p{0.10\linewidth} r r@{}}
\toprule
Comparison & Metric & Alt. & Estimate with 95\% CI & $p$ & $p_\mathrm{adj}$ & $n$ & $n_\mathrm{boot}$ \\
\midrule
Gaussian candidates, iCEM loop with QPP minus MC vs Gaussian candidates, carryover only with QPP minus MC & CVaR & two-sided & \asymci{0.4676}{0.3323}{0.3373} & 0.0042 & 0.0126 & 200 & 10000 \\
Gaussian candidates, iCEM loop with QPP minus MC vs Gaussian candidates, schedule only with QPP minus MC & CVaR & two-sided & \asymci{-0.3138}{0.8712}{0.6641} & 0.402 & 0.402 & 200 & 10000 \\
Gaussian candidates, carryover only+Gaussian candidates, schedule only-Gaussian candidates, iCEM loop with QPP minus MC (additivity) & CVaR & two-sided & \asymci{0.6099}{0.59506}{0.7981} & 0.044 & 0.08799 & 200 & 10000 \\
\bottomrule
\end{tabular}
 \vspace{0.5em}
\begin{tabular}{@{}>{\raggedright\arraybackslash}p{0.30\linewidth} >{\raggedright\arraybackslash}p{0.10\linewidth} >{\raggedright\arraybackslash}p{0.07\linewidth} >{\centering\arraybackslash}p{0.15\linewidth} >{\raggedleft\arraybackslash}p{0.10\linewidth} >{\raggedleft\arraybackslash}p{0.10\linewidth} r r@{}}
\toprule
Comparison & Metric & Alt. & Estimate with 95\% CI & $p$ & $p_\mathrm{adj}$ & $n$ & $n_\mathrm{boot}$ \\
\midrule
Gaussian candidates, iCEM loop with QPP minus MC vs Gaussian candidates, carryover only with QPP minus MC & CVaR & two-sided & \asymci{25.93}{14.65}{12.38} & 0.0009999 & 0.003 & 200 & 10000 \\
Gaussian candidates, iCEM loop with QPP minus MC vs Gaussian candidates, schedule only with QPP minus MC & CVaR & two-sided & \asymci{-0.2614}{17.1286}{14.5414} & 0.9991 & 1 & 200 & 10000 \\
Gaussian candidates, carryover only+Gaussian candidates, schedule only-Gaussian candidates, iCEM loop with QPP minus MC (additivity) & CVaR & two-sided & \asymci{2.687}{15.477}{18.813} & 0.7943 & 1 & 200 & 10000 \\
\bottomrule
\end{tabular}
 \caption{Loop decomposition on the QPP-minus-Monte-Carlo interaction scale
(overload above, inventory below): the iCEM loop control minus the
carryover-only control, minus the schedule-only control, and the additivity
residual (carryover plus schedule minus the iCEM loop, each referenced to
the Gaussian row). Exploratory, one Holm family per environment.}
\label{tab:loop-decomposition}
\end{sidewaystable}

\FloatBarrier
\subsection{Exploratory two-by-two and CityLearn surrogate tables}
\label{sec:appendix-qmc-citylearn}

This section collects diagnostics that are intentionally outside the frozen
primary gate families. The first two tables vary the candidate sampler and
evaluator in a controlled $2\times2$ design; the third records the simulator-
surrogate negative result.

\begin{sidewaystable}[p]
\centering
\footnotesize
\setlength{\tabcolsep}{3pt}

\begin{tabular}{@{}>{\raggedright\arraybackslash}p{0.30\linewidth} >{\raggedright\arraybackslash}p{0.10\linewidth} >{\raggedright\arraybackslash}p{0.07\linewidth} >{\centering\arraybackslash}p{0.15\linewidth} >{\raggedleft\arraybackslash}p{0.10\linewidth} >{\raggedleft\arraybackslash}p{0.10\linewidth} r r@{}}
\toprule
Comparison & Metric & Alt. & Estimate with 95\% CI & $p$ & $p_\mathrm{adj}$ & $n$ & $n_\mathrm{boot}$ \\
\midrule
interaction & mean & two-sided & \asymci{0.01482}{0.09470}{0.08788} & 0.7079 & -- & 200 & 10000 \\
interaction & CVaR non-additivity & two-sided & \asymci{-0.3687}{0.9353}{0.8690} & 0.4344 & -- & 200 & 10000 \\
\bottomrule
\end{tabular}
\vspace{0.6em}

\begin{tabular}{@{}>{\raggedright\arraybackslash}p{0.30\linewidth} >{\raggedright\arraybackslash}p{0.10\linewidth} >{\raggedright\arraybackslash}p{0.07\linewidth} >{\centering\arraybackslash}p{0.15\linewidth} >{\raggedleft\arraybackslash}p{0.10\linewidth} >{\raggedleft\arraybackslash}p{0.10\linewidth} r r@{}}
\toprule
Comparison & Metric & Alt. & Estimate with 95\% CI & $p$ & $p_\mathrm{adj}$ & $n$ & $n_\mathrm{boot}$ \\
\midrule
Gaussian candidates with QPP vs Gaussian candidates with MC & mean & two-sided & \asymci{1.851}{0.460}{0.339} & $9.999\times10^{-5}$ & 0.0005999 & 200 & 10000 \\
Gaussian candidates with QPP vs Gaussian candidates with MC & CVaR & two-sided & \asymci{-2.872}{4.619}{3.2842} & 0.1188 & 0.4752 & 200 & 10000 \\
QMC candidates with QPP vs Gaussian candidates with MC & mean & two-sided & \asymci{1.845}{0.468}{0.333} & $9.999\times10^{-5}$ & 0.0005999 & 200 & 10000 \\
QMC candidates with QPP vs Gaussian candidates with MC & CVaR & two-sided & \asymci{-2.731}{4.458}{3.2206} & 0.1244 & 0.4752 & 200 & 10000 \\
QMC candidates with MC vs Gaussian candidates with MC & mean & two-sided & \asymci{-0.02092}{0.06523}{0.08891} & 0.5697 & 0.5697 & 200 & 10000 \\
QMC candidates with MC vs Gaussian candidates with MC & CVaR & two-sided & \asymci{0.5093}{0.6306}{0.8677} & 0.1728 & 0.4752 & 200 & 10000 \\
\bottomrule
\end{tabular}
\vspace{0.6em}

 \caption{Exploratory $2 \times 2$ study on overload: quasi-Monte-Carlo
against Gaussian candidates crossed with QPP against $S = 5$ Monte Carlo,
interaction contrasts and comparisons against the Gaussian Monte Carlo
cell. The overload C3 margin of the composition study is half this table's
Gaussian QPP-minus-Monte-Carlo $\CVaR_{0.1}$ gap.}
\label{tab:qmc}
\end{sidewaystable}

\begin{sidewaystable}[p]
\centering
\footnotesize
\setlength{\tabcolsep}{3pt}

\begin{tabular}{@{}>{\raggedright\arraybackslash}p{0.30\linewidth} >{\raggedright\arraybackslash}p{0.10\linewidth} >{\raggedright\arraybackslash}p{0.07\linewidth} >{\centering\arraybackslash}p{0.15\linewidth} >{\raggedleft\arraybackslash}p{0.10\linewidth} >{\raggedleft\arraybackslash}p{0.10\linewidth} r r@{}}
\toprule
Comparison & Metric & Alt. & Estimate with 95\% CI & $p$ & $p_\mathrm{adj}$ & $n$ & $n_\mathrm{boot}$ \\
\midrule
interaction & mean & two-sided & \asymci{2.083}{1.6884}{1.543} & 0.016 & -- & 200 & 10000 \\
interaction & CVaR non-additivity & two-sided & \asymci{-1.73}{15.20}{11.168} & 0.6833 & -- & 200 & 10000 \\
\bottomrule
\end{tabular}
\vspace{0.6em}

\begin{tabular}{@{}>{\raggedright\arraybackslash}p{0.30\linewidth} >{\raggedright\arraybackslash}p{0.10\linewidth} >{\raggedright\arraybackslash}p{0.07\linewidth} >{\centering\arraybackslash}p{0.15\linewidth} >{\raggedleft\arraybackslash}p{0.10\linewidth} >{\raggedleft\arraybackslash}p{0.10\linewidth} r r@{}}
\toprule
Comparison & Metric & Alt. & Estimate with 95\% CI & $p$ & $p_\mathrm{adj}$ & $n$ & $n_\mathrm{boot}$ \\
\midrule
Gaussian candidates with QPP vs Gaussian candidates with MC & mean & two-sided & \asymci{102}{9.78}{7.8} & $9.999\times10^{-5}$ & 0.0005999 & 200 & 10000 \\
Gaussian candidates with QPP vs Gaussian candidates with MC & CVaR & two-sided & \asymci{16.97}{72.53}{56.41} & 0.5471 & 1 & 200 & 10000 \\
QMC candidates with QPP vs Gaussian candidates with MC & mean & two-sided & \asymci{104.9}{10.05}{7.8} & $9.999\times10^{-5}$ & 0.0005999 & 200 & 10000 \\
QMC candidates with QPP vs Gaussian candidates with MC & CVaR & two-sided & \asymci{19.49}{74.86}{56.99} & 0.4978 & 1 & 200 & 10000 \\
QMC candidates with MC vs Gaussian candidates with MC & mean & two-sided & \asymci{0.8081}{1.5400}{1.6529} & 0.3249 & 1 & 200 & 10000 \\
QMC candidates with MC vs Gaussian candidates with MC & CVaR & two-sided & \asymci{4.246}{10.989}{15.114} & 0.4206 & 1 & 200 & 10000 \\
\bottomrule
\end{tabular}
\vspace{0.6em}

 \caption{Exploratory $2 \times 2$ study on inventory: the same layout as
Table~\ref{tab:qmc}. The inventory C3 margin of the composition study is half
this table's Gaussian QPP-minus-Monte-Carlo $\CVaR_{0.1}$ gap.}
\label{tab:qmc-inventory}
\end{sidewaystable}

\begin{sidewaystable}[p]
\centering
\footnotesize
\setlength{\tabcolsep}{3pt}

\begin{tabular}{@{}>{\raggedright\arraybackslash}p{0.30\linewidth} >{\raggedright\arraybackslash}p{0.10\linewidth} >{\raggedright\arraybackslash}p{0.07\linewidth} >{\centering\arraybackslash}p{0.15\linewidth} >{\raggedleft\arraybackslash}p{0.10\linewidth} >{\raggedleft\arraybackslash}p{0.10\linewidth} r r@{}}
\toprule
Comparison & Metric & Alt. & Estimate with 95\% CI & $p$ & $p_\mathrm{adj}$ & $n$ & $n_\mathrm{boot}$ \\
\midrule
QPP lambda=0.5 vs MC (S=5) & mean & two-sided & \asymci{-1.616}{0.175}{0.175} & $9.999\times10^{-5}$ & 0.0007999 & 200 & 10000 \\
QPP lambda=0.5 vs MC (S=5) & CVaR & two-sided & \asymci{-2.822}{0.526}{0.558} & 0.0002 & 0.0009999 & 200 & 10000 \\
mean vs MC (S=5) & mean & two-sided & \asymci{-1.471}{0.162}{0.160} & $9.999\times10^{-5}$ & 0.0007999 & 200 & 10000 \\
mean vs MC (S=5) & CVaR & two-sided & \asymci{-2.614}{0.511}{0.513} & 0.0002 & 0.0009999 & 200 & 10000 \\
\bottomrule
\end{tabular}
\vspace{0.6em}

\begin{tabular}{@{}>{\raggedright\arraybackslash}p{0.30\linewidth} >{\raggedright\arraybackslash}p{0.10\linewidth} >{\raggedright\arraybackslash}p{0.07\linewidth} >{\centering\arraybackslash}p{0.15\linewidth} >{\raggedleft\arraybackslash}p{0.10\linewidth} >{\raggedleft\arraybackslash}p{0.10\linewidth} r r@{}}
\toprule
Comparison & Metric & Alt. & Estimate with 95\% CI & $p$ & $p_\mathrm{adj}$ & $n$ & $n_\mathrm{boot}$ \\
\midrule
QPP lambda=0.5 vs mean & mean & two-sided & \asymci{-0.1452}{0.0679}{0.0670} & 0.0002 & 0.0009999 & 200 & 10000 \\
QPP lambda=0.5 vs mean & CVaR & two-sided & \asymci{-0.2079}{0.1470}{0.15916} & 0.0104 & 0.0104 & 200 & 10000 \\
MC (S=5) vs mean & mean & two-sided & \asymci{1.471}{0.160}{0.162} & $9.999\times10^{-5}$ & 0.0007999 & 200 & 10000 \\
MC (S=5) vs mean & CVaR & two-sided & \asymci{2.614}{0.513}{0.511} & 0.0002 & 0.0009999 & 200 & 10000 \\
\bottomrule
\end{tabular}
\vspace{0.6em}

 \caption{CityLearn surrogate (plan on a fitted diagonal-monotone linear
quantile surrogate, execute on the real environment): QPP at $\lam = 0.5$,
mean rollout and $S = 5$ Monte Carlo at matched $\Nevals$, $n = 200$. One
Holm family over both blocks.}
\label{tab:citylearn}
\end{sidewaystable}

\FloatBarrier
\subsection{Microgrid: H1 in full, and versus budget-matched sampling}
\label{sec:appendix-microgrid-h1}

This section keeps three comparisons separate: QPP against mean rollout, the
exploratory operational and multiplicity audits, and QPP against
budget-matched Monte Carlo. Each block is completed before the next begins.

\subsubsection{Comparison with mean rollout}

\paragraph{Against a constant-reserve oracle.} QPP over-reserves against
the CVaR-optimal constant reserve by about a factor of two (planned reserve
$0.67$--$0.78$ across the four sites against $0.275$--$0.375$), the coupling's
conservatism at $H = 24$, yet its realized CVaR beats the best member of that
one-parameter family ($6.882$ against $7.5361$ at site $242$) because the free
planner shapes the whole $24$-step profile. The oracle is a restricted-family
benchmark, not an upper bound, and its $200{,}000$-sample estimate is not the
same estimator as a $200$-repeat realized CVaR.

\begin{sidewaystable}[p]
\centering
\footnotesize
\setlength{\tabcolsep}{3pt}
\begin{tabular}{@{}>{\raggedright\arraybackslash}p{0.30\linewidth} >{\raggedright\arraybackslash}p{0.10\linewidth} >{\raggedright\arraybackslash}p{0.07\linewidth} >{\centering\arraybackslash}p{0.15\linewidth} >{\raggedleft\arraybackslash}p{0.10\linewidth} >{\raggedleft\arraybackslash}p{0.10\linewidth} r r@{}}
\toprule
Comparison & Metric & Alt. & Estimate with 95\% CI & $p$ & $p_\mathrm{adj}$ & $n$ & $n_\mathrm{boot}$ \\
\midrule
QPP lambda=0 vs mean & mean & two-sided & \asymci{-0.06466}{0.10444}{0.08902} & 0.201 & 1 & 200 & 10000 \\
QPP lambda=0 vs mean & CVaR & lower & \asymci{-1.598}{\infty}{0.569} & $9.999\times10^{-5}$ & 0.0012 & 200 & 10000 \\
QPP lambda=0 vs mean & V1 headroom loss & lower & \asymci{-0.52}{\infty}{0.06} & $4.93\times10^{-32}$ & $6.409\times10^{-31}$ & 200 & 10000 \\
QPP lambda=0.5 vs mean & mean & two-sided & \asymci{-0.05032}{0.11038}{0.09205} & 0.3521 & 1 & 200 & 10000 \\
QPP lambda=0.5 vs mean & CVaR & lower & \asymci{-1.597}{\infty}{0.592} & $9.999\times10^{-5}$ & 0.0012 & 200 & 10000 \\
QPP lambda=0.5 vs mean & V1 headroom loss & lower & \asymci{-0.525}{\infty}{0.060} & $2.465\times10^{-32}$ & $3.698\times10^{-31}$ & 200 & 10000 \\
QPP lambda=1 vs mean & mean & two-sided & \asymci{-0.04491}{0.11529}{0.09347} & 0.4293 & 1 & 200 & 10000 \\
QPP lambda=1 vs mean & CVaR & lower & \asymci{-1.619}{\infty}{0.618} & $9.999\times10^{-5}$ & 0.0012 & 200 & 10000 \\
QPP lambda=1 vs mean & V1 headroom loss & lower & \asymci{-0.525}{\infty}{0.060} & $2.465\times10^{-32}$ & $3.698\times10^{-31}$ & 200 & 10000 \\
moment match vs mean & mean & two-sided & \asymci{0.001425}{0.0009397}{0.001119} & 0.0016 & 0.0128 & 200 & 10000 \\
moment match vs mean & CVaR & two-sided & \asymci{0.003675}{0.006220}{0.007105} & 0.2602 & 1 & 200 & 10000 \\
moment match vs mean & V1 headroom loss & two-sided & \asymci{0}{0}{0} & 1 & 1 & 200 & 10000 \\
MC (S=5) vs mean & mean & two-sided & \asymci{0.04449}{0.01873}{0.01586} & $9.999\times10^{-5}$ & 0.0012 & 200 & 10000 \\
MC (S=5) vs mean & CVaR & two-sided & \asymci{-0.07031}{0.13759}{0.11055} & 0.26 & 1 & 200 & 10000 \\
MC (S=5) vs mean & V1 headroom loss & two-sided & \asymci{-0.03}{0.025}{0.02} & 0.03125 & 0.2188 & 200 & 10000 \\
\bottomrule
\end{tabular}
 \caption{Microgrid H1, every row of the comparison against mean rollout,
evaluation site $242$ (sites $253$, $255$, $279$ follow): every $\lam$
value and moment matching, alongside the $\lam = 0.5$ row cited in
Section~\ref{sec:results-microgrid}. One Holm family per site.}
\label{tab:microgrid-h1-full-242}
\end{sidewaystable}

\begin{sidewaystable}[p]
\centering
\footnotesize
\setlength{\tabcolsep}{3pt}
\begin{tabular}{@{}>{\raggedright\arraybackslash}p{0.30\linewidth} >{\raggedright\arraybackslash}p{0.10\linewidth} >{\raggedright\arraybackslash}p{0.07\linewidth} >{\centering\arraybackslash}p{0.15\linewidth} >{\raggedleft\arraybackslash}p{0.10\linewidth} >{\raggedleft\arraybackslash}p{0.10\linewidth} r r@{}}
\toprule
Comparison & Metric & Alt. & Estimate with 95\% CI & $p$ & $p_\mathrm{adj}$ & $n$ & $n_\mathrm{boot}$ \\
\midrule
QPP lambda=0 vs mean & mean & two-sided & \asymci{0.02203}{0.08263}{0.06923} & 0.5827 & 1 & 200 & 10000 \\
QPP lambda=0 vs mean & CVaR & lower & \asymci{-1.175}{\infty}{0.4754} & $9.999\times10^{-5}$ & 0.0012 & 200 & 10000 \\
QPP lambda=0 vs mean & V1 headroom loss & lower & \asymci{-0.48}{\infty}{0.06} & $1.262\times10^{-29}$ & $1.893\times10^{-28}$ & 200 & 10000 \\
QPP lambda=0.5 vs mean & mean & two-sided & \asymci{0.03054}{0.08701}{0.07226} & 0.4824 & 1 & 200 & 10000 \\
QPP lambda=0.5 vs mean & CVaR & lower & \asymci{-1.184}{\infty}{0.4911} & $9.999\times10^{-5}$ & 0.0012 & 200 & 10000 \\
QPP lambda=0.5 vs mean & V1 headroom loss & lower & \asymci{-0.475}{\infty}{0.060} & $2.524\times10^{-29}$ & $3.534\times10^{-28}$ & 200 & 10000 \\
QPP lambda=1 vs mean & mean & two-sided & \asymci{0.03795}{0.08723}{0.07075} & 0.3681 & 1 & 200 & 10000 \\
QPP lambda=1 vs mean & CVaR & lower & \asymci{-1.163}{\infty}{0.4825} & $9.999\times10^{-5}$ & 0.0012 & 200 & 10000 \\
QPP lambda=1 vs mean & V1 headroom loss & lower & \asymci{-0.475}{\infty}{0.060} & $2.524\times10^{-29}$ & $3.534\times10^{-28}$ & 200 & 10000 \\
moment match vs mean & mean & two-sided & \asymci{0.0008595}{0.0004430}{0.0005905} & $9.999\times10^{-5}$ & 0.0012 & 200 & 10000 \\
moment match vs mean & CVaR & two-sided & \asymci{0.001043}{0.00109285}{0.001530} & 0.06539 & 0.4578 & 200 & 10000 \\
moment match vs mean & V1 headroom loss & two-sided & \asymci{0}{0}{0} & 1 & 1 & 200 & 10000 \\
MC (S=5) vs mean & mean & two-sided & \asymci{0.03847}{0.01316}{0.01167} & $9.999\times10^{-5}$ & 0.0012 & 200 & 10000 \\
MC (S=5) vs mean & CVaR & two-sided & \asymci{-0.07157}{0.09903}{0.09560} & 0.168 & 1 & 200 & 10000 \\
MC (S=5) vs mean & V1 headroom loss & two-sided & \asymci{-0.01}{0.015}{0.01} & 0.5 & 1 & 200 & 10000 \\
\bottomrule
\end{tabular}
 \caption{Microgrid H1, evaluation site $253$. One Holm family per site.}
\label{tab:microgrid-h1-full-253}
\end{sidewaystable}

\begin{sidewaystable}[p]
\centering
\footnotesize
\setlength{\tabcolsep}{3pt}
\begin{tabular}{@{}>{\raggedright\arraybackslash}p{0.30\linewidth} >{\raggedright\arraybackslash}p{0.10\linewidth} >{\raggedright\arraybackslash}p{0.07\linewidth} >{\centering\arraybackslash}p{0.15\linewidth} >{\raggedleft\arraybackslash}p{0.10\linewidth} >{\raggedleft\arraybackslash}p{0.10\linewidth} r r@{}}
\toprule
Comparison & Metric & Alt. & Estimate with 95\% CI & $p$ & $p_\mathrm{adj}$ & $n$ & $n_\mathrm{boot}$ \\
\midrule
QPP lambda=0 vs mean & mean & two-sided & \asymci{0.08332}{0.075724}{0.06128} & 0.0108 & 0.06419 & 200 & 10000 \\
QPP lambda=0 vs mean & CVaR & lower & \asymci{-0.901}{\infty}{0.4547} & $9.999\times10^{-5}$ & 0.0012 & 200 & 10000 \\
QPP lambda=0 vs mean & V1 headroom loss & lower & \asymci{-0.36}{\infty}{0.055} & $2.118\times10^{-22}$ & $3.176\times10^{-21}$ & 200 & 10000 \\
QPP lambda=0.5 vs mean & mean & two-sided & \asymci{0.09059}{0.08030}{0.06461} & 0.0107 & 0.06419 & 200 & 10000 \\
QPP lambda=0.5 vs mean & CVaR & lower & \asymci{-0.9333}{\infty}{0.4788} & $9.999\times10^{-5}$ & 0.0012 & 200 & 10000 \\
QPP lambda=0.5 vs mean & V1 headroom loss & lower & \asymci{-0.36}{\infty}{0.055} & $2.118\times10^{-22}$ & $3.176\times10^{-21}$ & 200 & 10000 \\
QPP lambda=1 vs mean & mean & two-sided & \asymci{0.09761}{0.07749}{0.06199} & 0.003 & 0.021 & 200 & 10000 \\
QPP lambda=1 vs mean & CVaR & lower & \asymci{-0.9025}{\infty}{0.4579} & $9.999\times10^{-5}$ & 0.0012 & 200 & 10000 \\
QPP lambda=1 vs mean & V1 headroom loss & lower & \asymci{-0.355}{\infty}{0.055} & $4.235\times10^{-22}$ & $5.506\times10^{-21}$ & 200 & 10000 \\
moment match vs mean & mean & two-sided & \asymci{0.0006958}{0.0005657}{0.0010522} & $9.999\times10^{-5}$ & 0.0012 & 200 & 10000 \\
moment match vs mean & CVaR & two-sided & \asymci{0.004711}{0.004711}{0.009339} & 0.08279 & 0.3312 & 200 & 10000 \\
moment match vs mean & V1 headroom loss & two-sided & \asymci{0}{0}{0} & 1 & 1 & 200 & 10000 \\
MC (S=5) vs mean & mean & two-sided & \asymci{0.04491}{0.01019}{0.01139} & $9.999\times10^{-5}$ & 0.0012 & 200 & 10000 \\
MC (S=5) vs mean & CVaR & two-sided & \asymci{0.07551}{0.084803}{0.08359} & 0.08699 & 0.3312 & 200 & 10000 \\
MC (S=5) vs mean & V1 headroom loss & two-sided & \asymci{-0.01}{0.015}{0.01} & 0.5 & 1 & 200 & 10000 \\
\bottomrule
\end{tabular}
 \caption{Microgrid H1, evaluation site $255$. One Holm family per site.}
\label{tab:microgrid-h1-full-255}
\end{sidewaystable}

\begin{sidewaystable}[p]
\centering
\footnotesize
\setlength{\tabcolsep}{3pt}
\begin{tabular}{@{}>{\raggedright\arraybackslash}p{0.30\linewidth} >{\raggedright\arraybackslash}p{0.10\linewidth} >{\raggedright\arraybackslash}p{0.07\linewidth} >{\centering\arraybackslash}p{0.15\linewidth} >{\raggedleft\arraybackslash}p{0.10\linewidth} >{\raggedleft\arraybackslash}p{0.10\linewidth} r r@{}}
\toprule
Comparison & Metric & Alt. & Estimate with 95\% CI & $p$ & $p_\mathrm{adj}$ & $n$ & $n_\mathrm{boot}$ \\
\midrule
QPP lambda=0 vs mean & mean & two-sided & \asymci{-0.1099}{0.1239}{0.104454} & 0.05839 & 0.3504 & 200 & 10000 \\
QPP lambda=0 vs mean & CVaR & lower & \asymci{-1.874}{\infty}{0.691} & $9.999\times10^{-5}$ & 0.0012 & 200 & 10000 \\
QPP lambda=0 vs mean & V1 headroom loss & lower & \asymci{-0.47}{\infty}{0.055} & $5.049\times10^{-29}$ & $7.573\times10^{-28}$ & 200 & 10000 \\
QPP lambda=0.5 vs mean & mean & two-sided & \asymci{-0.09605}{0.12485}{0.104545} & 0.1005 & 0.5024 & 200 & 10000 \\
QPP lambda=0.5 vs mean & CVaR & lower & \asymci{-1.828}{\infty}{0.667} & $9.999\times10^{-5}$ & 0.0012 & 200 & 10000 \\
QPP lambda=0.5 vs mean & V1 headroom loss & lower & \asymci{-0.465}{\infty}{0.060} & $1.01\times10^{-28}$ & $1.414\times10^{-27}$ & 200 & 10000 \\
QPP lambda=1 vs mean & mean & two-sided & \asymci{-0.09015}{0.12385}{0.10282} & 0.1251 & 0.5024 & 200 & 10000 \\
QPP lambda=1 vs mean & CVaR & lower & \asymci{-1.8}{\infty}{0.66} & $9.999\times10^{-5}$ & 0.0012 & 200 & 10000 \\
QPP lambda=1 vs mean & V1 headroom loss & lower & \asymci{-0.465}{\infty}{0.060} & $1.01\times10^{-28}$ & $1.414\times10^{-27}$ & 200 & 10000 \\
moment match vs mean & mean & two-sided & \asymci{0.0003742}{0.000383772}{0.0003506} & 0.0369 & 0.2583 & 200 & 10000 \\
moment match vs mean & CVaR & two-sided & \asymci{-0.001084}{0.002650}{0.002120} & 0.6551 & 1 & 200 & 10000 \\
moment match vs mean & V1 headroom loss & two-sided & \asymci{0}{0}{0} & 1 & 1 & 200 & 10000 \\
MC (S=5) vs mean & mean & two-sided & \asymci{0.04156}{0.03134}{0.02484} & 0.0016 & 0.0144 & 200 & 10000 \\
MC (S=5) vs mean & CVaR & two-sided & \asymci{-0.1726}{0.2794}{0.22801} & 0.1656 & 0.5024 & 200 & 10000 \\
MC (S=5) vs mean & V1 headroom loss & two-sided & \asymci{-0.05}{0.03}{0.025} & 0.001953 & 0.01562 & 200 & 10000 \\
\bottomrule
\end{tabular}
 \caption{Microgrid H1, evaluation site $279$. One Holm family per site.}
\label{tab:microgrid-h1-full-279}
\end{sidewaystable}

\FloatBarrier
\subsubsection{Operational endpoint and multiplicity audit}

\paragraph{Actual import on the recorded execution stream.} The following
summary reports episode-level actual-import exceedance percentages with exact
binomial $95\%$ intervals for all six planners. The final comparison reports
QPP at $\lam=0.5$ minus mean rollout in percentage points together with its
one-sided $95\%$ upper bound. The corresponding $p$-values are adjusted by
Holm's procedure across the four sitewise contrasts. The endpoint remains
exploratory, but its replay now uses the identical execution seed sequence as
the recorded cost and $V1$ observations.
\begin{sidewaystable}[p]
\centering
\footnotesize
\setlength{\tabcolsep}{3pt}
\begin{tabular}{@{}>{\centering\arraybackslash}p{0.04\linewidth} *{6}{>{\centering\arraybackslash}p{0.095\linewidth}} >{\centering\arraybackslash}p{0.18\linewidth} >{\raggedleft\arraybackslash}p{0.08\linewidth}@{}}
\toprule
Site & QPP $\lambda=0$ & QPP $\lambda=0.5$ & QPP $\lambda=1$ & Mean & MM & MC $S=5$ & QPP$_{0.5}$--Mean with upper bound & $p_{\mathrm{Holm}}$ \\
\midrule
242 & \asymci{33.0}{6.5}{7.0} & \asymci{33.5}{6.5}{7.0} & \asymci{34.5}{6.6}{7.0} & \asymci{62.5}{7.1}{6.7} & \asymci{62.5}{7.1}{6.7} & \asymci{65.5}{7.0}{6.6} & \upperci{-29.0}{5.0} & 0.0004 \\
253 & \asymci{28.5}{6.1}{6.8} & \asymci{29.0}{6.2}{6.8} & \asymci{29.0}{6.2}{6.8} & \asymci{55.5}{7.2}{7.0} & \asymci{55.5}{7.2}{7.0} & \asymci{58.5}{7.2}{6.9} & \upperci{-26.5}{5.0} & 0.0004 \\
255 & \asymci{27.0}{6.0}{6.7} & \asymci{28.5}{6.1}{6.8} & \asymci{28.5}{6.1}{6.8} & \asymci{43.0}{7.0}{7.2} & \asymci{43.5}{7.0}{7.2} & \asymci{44.5}{7.0}{7.2} & \upperci{-14.5}{4.5} & 0.0004 \\
279 & \asymci{36.5}{6.7}{7.1} & \asymci{37.5}{6.7}{7.1} & \asymci{37.0}{6.7}{7.1} & \asymci{60.0}{7.1}{6.8} & \asymci{60.0}{7.1}{6.8} & \asymci{61.5}{7.1}{6.8} & \upperci{-22.5}{5.0} & 0.0004 \\
\bottomrule
\end{tabular}
 \caption{Episode-level actual-import exceedance rates on the recorded
execution stream. Intervals are exact binomial $95\%$ intervals; the final
columns compare QPP at $\lam=0.5$ with mean rollout. This endpoint is
exploratory.}
\label{tab:microgrid-actual-import}
\end{sidewaystable}

\paragraph{Cross-site B1 multiplicity audit.}
The frozen gate remains the four unadjusted sitewise lower bounds. The table
below is a post-hoc robustness audit of the same paired arrays. It adds a
null-centered one-sided test against the actual $0.10$ bar, applies Holm's
procedure across sites, and reports the conservative $98.75\%$ lower bound
that gives Bonferroni simultaneous $95\%$ coverage.
\begin{table}[h]
\centering
\footnotesize
\begin{tabular}{lrrrrr}
\toprule
Site & Gain & 95\% lower & 98.75\% lower & $p_{0.10}$ & Holm $p$ \\
\midrule
242 & 0.188 & 0.136 & 0.118 & 0.0004 & 0.0013 \\
253 & 0.153 & 0.102 & 0.085 & 0.0217 & 0.0435 \\
255 & 0.121 & 0.068 & 0.052 & 0.2097 & 0.2097 \\
279 & 0.185 & 0.141 & 0.125 & 0.0000 & 0.0001 \\
\bottomrule
\end{tabular}
 \caption{Post-hoc cross-site audit of the B1 relative $\CVaR_{0.1}$ gain.
The $95\%$ column reproduces the sitewise gate scale; the $98.75\%$ column
and Holm $p$ values address the four-site family in two complementary ways.}
\label{tab:cross-site-b1-audit}
\end{table}

\FloatBarrier
\subsubsection{Comparison with budget-matched Monte Carlo}

The following four tables compare every planner with $S=5$ Monte Carlo at
matched dynamics-evaluation cost. They use one Holm family per site and are
kept in site order.

\begin{sidewaystable}[p]
\centering
\footnotesize
\setlength{\tabcolsep}{3pt}
\begin{tabular}{@{}>{\raggedright\arraybackslash}p{0.30\linewidth} >{\raggedright\arraybackslash}p{0.10\linewidth} >{\raggedright\arraybackslash}p{0.07\linewidth} >{\centering\arraybackslash}p{0.15\linewidth} >{\raggedleft\arraybackslash}p{0.10\linewidth} >{\raggedleft\arraybackslash}p{0.10\linewidth} r r@{}}
\toprule
Comparison & Metric & Alt. & Estimate with 95\% CI & $p$ & $p_\mathrm{adj}$ & $n$ & $n_\mathrm{boot}$ \\
\midrule
QPP lambda=0 vs MC (S=5) & mean & two-sided & \asymci{-0.1091}{0.0984}{0.08493} & 0.0192 & 0.1344 & 200 & 10000 \\
QPP lambda=0 vs MC (S=5) & CVaR & two-sided & \asymci{-1.527}{0.713}{0.6356} & 0.0002 & 0.002 & 200 & 10000 \\
QPP lambda=0 vs MC (S=5) & V1 headroom loss & two-sided & \asymci{-0.49}{0.07}{0.07} & $6.311\times10^{-30}$ & $8.204\times10^{-29}$ & 200 & 10000 \\
QPP lambda=0.5 vs MC (S=5) & mean & two-sided & \asymci{-0.09481}{0.10519}{0.087899} & 0.05339 & 0.2136 & 200 & 10000 \\
QPP lambda=0.5 vs MC (S=5) & CVaR & two-sided & \asymci{-1.527}{0.757}{0.6683} & 0.0002 & 0.002 & 200 & 10000 \\
QPP lambda=0.5 vs MC (S=5) & V1 headroom loss & two-sided & \asymci{-0.495}{0.070}{0.070} & $3.155\times10^{-30}$ & $4.733\times10^{-29}$ & 200 & 10000 \\
QPP lambda=1 vs MC (S=5) & mean & two-sided & \asymci{-0.0894}{0.1102}{0.0903266} & 0.08169 & 0.2451 & 200 & 10000 \\
QPP lambda=1 vs MC (S=5) & CVaR & two-sided & \asymci{-1.549}{0.824}{0.6948} & 0.0002 & 0.002 & 200 & 10000 \\
QPP lambda=1 vs MC (S=5) & V1 headroom loss & two-sided & \asymci{-0.495}{0.070}{0.070} & $3.155\times10^{-30}$ & $4.733\times10^{-29}$ & 200 & 10000 \\
mean vs MC (S=5) & mean & two-sided & \asymci{-0.04449}{0.01642}{0.01857} & $9.999\times10^{-5}$ & 0.0012 & 200 & 10000 \\
mean vs MC (S=5) & CVaR & two-sided & \asymci{0.07031}{0.10902}{0.13999} & 0.2522 & 0.4544 & 200 & 10000 \\
mean vs MC (S=5) & V1 headroom loss & two-sided & \asymci{0.03}{0.02}{0.025} & 0.03125 & 0.1875 & 200 & 10000 \\
moment match vs MC (S=5) & mean & two-sided & \asymci{-0.04306}{0.01595}{0.01877} & $9.999\times10^{-5}$ & 0.0012 & 200 & 10000 \\
moment match vs MC (S=5) & CVaR & two-sided & \asymci{0.07398}{0.10949}{0.13702} & 0.2272 & 0.4544 & 200 & 10000 \\
moment match vs MC (S=5) & V1 headroom loss & two-sided & \asymci{0.03}{0.02}{0.025} & 0.03125 & 0.1875 & 200 & 10000 \\
\bottomrule
\end{tabular}
 \caption{Microgrid H1, every planner against $S = 5$ Monte Carlo at
matched $\Nevals$, two-sided, evaluation site $242$ (sites $253$, $255$,
$279$ follow). One Holm family per site.}
\label{tab:microgrid-vs-mc-242}
\end{sidewaystable}

\begin{sidewaystable}[p]
\centering
\footnotesize
\setlength{\tabcolsep}{3pt}
\begin{tabular}{@{}>{\raggedright\arraybackslash}p{0.30\linewidth} >{\raggedright\arraybackslash}p{0.10\linewidth} >{\raggedright\arraybackslash}p{0.07\linewidth} >{\centering\arraybackslash}p{0.15\linewidth} >{\raggedleft\arraybackslash}p{0.10\linewidth} >{\raggedleft\arraybackslash}p{0.10\linewidth} r r@{}}
\toprule
Comparison & Metric & Alt. & Estimate with 95\% CI & $p$ & $p_\mathrm{adj}$ & $n$ & $n_\mathrm{boot}$ \\
\midrule
QPP lambda=0 vs MC (S=5) & mean & two-sided & \asymci{-0.01644}{0.07776}{0.06589} & 0.6672 & 1 & 200 & 10000 \\
QPP lambda=0 vs MC (S=5) & CVaR & two-sided & \asymci{-1.104}{0.578}{0.499} & 0.0002 & 0.002 & 200 & 10000 \\
QPP lambda=0 vs MC (S=5) & V1 headroom loss & two-sided & \asymci{-0.47}{0.07}{0.07} & $1.01\times10^{-28}$ & $1.515\times10^{-27}$ & 200 & 10000 \\
QPP lambda=0.5 vs MC (S=5) & mean & two-sided & \asymci{-0.007932}{0.083038}{0.069022} & 0.8516 & 1 & 200 & 10000 \\
QPP lambda=0.5 vs MC (S=5) & CVaR & two-sided & \asymci{-1.113}{0.617}{0.529} & 0.0002 & 0.002 & 200 & 10000 \\
QPP lambda=0.5 vs MC (S=5) & V1 headroom loss & two-sided & \asymci{-0.465}{0.070}{0.070} & $2.019\times10^{-28}$ & $2.827\times10^{-27}$ & 200 & 10000 \\
QPP lambda=1 vs MC (S=5) & mean & two-sided & \asymci{-0.0005184}{0.0818216}{0.0678084} & 0.9902 & 1 & 200 & 10000 \\
QPP lambda=1 vs MC (S=5) & CVaR & two-sided & \asymci{-1.091}{0.627}{0.5155} & 0.0002 & 0.002 & 200 & 10000 \\
QPP lambda=1 vs MC (S=5) & V1 headroom loss & two-sided & \asymci{-0.465}{0.070}{0.070} & $2.019\times10^{-28}$ & $2.827\times10^{-27}$ & 200 & 10000 \\
mean vs MC (S=5) & mean & two-sided & \asymci{-0.03847}{0.01191}{0.01288} & $9.999\times10^{-5}$ & 0.0012 & 200 & 10000 \\
mean vs MC (S=5) & CVaR & two-sided & \asymci{0.07157}{0.09587}{0.09683} & 0.1614 & 1 & 200 & 10000 \\
mean vs MC (S=5) & V1 headroom loss & two-sided & \asymci{0.01}{0.01}{0.015} & 0.5 & 1 & 200 & 10000 \\
moment match vs MC (S=5) & mean & two-sided & \asymci{-0.03761}{0.01170}{0.01320} & $9.999\times10^{-5}$ & 0.0012 & 200 & 10000 \\
moment match vs MC (S=5) & CVaR & two-sided & \asymci{0.07262}{0.09572}{0.09958} & 0.1618 & 1 & 200 & 10000 \\
moment match vs MC (S=5) & V1 headroom loss & two-sided & \asymci{0.01}{0.01}{0.015} & 0.5 & 1 & 200 & 10000 \\
\bottomrule
\end{tabular}
 \caption{Microgrid H1 versus $S = 5$ Monte Carlo, evaluation site $253$. One
Holm family per site.}
\label{tab:microgrid-vs-mc-253}
\end{sidewaystable}

\begin{sidewaystable}[p]
\centering
\footnotesize
\setlength{\tabcolsep}{3pt}
\begin{tabular}{@{}>{\raggedright\arraybackslash}p{0.30\linewidth} >{\raggedright\arraybackslash}p{0.10\linewidth} >{\raggedright\arraybackslash}p{0.07\linewidth} >{\centering\arraybackslash}p{0.15\linewidth} >{\raggedleft\arraybackslash}p{0.10\linewidth} >{\raggedleft\arraybackslash}p{0.10\linewidth} r r@{}}
\toprule
Comparison & Metric & Alt. & Estimate with 95\% CI & $p$ & $p_\mathrm{adj}$ & $n$ & $n_\mathrm{boot}$ \\
\midrule
QPP lambda=0 vs MC (S=5) & mean & two-sided & \asymci{0.03842}{0.07949}{0.06378} & 0.3208 & 1 & 200 & 10000 \\
QPP lambda=0 vs MC (S=5) & CVaR & two-sided & \asymci{-0.9765}{0.6475}{0.5367} & 0.0002 & 0.002 & 200 & 10000 \\
QPP lambda=0 vs MC (S=5) & V1 headroom loss & two-sided & \asymci{-0.35}{0.065}{0.065} & $1.694\times10^{-21}$ & $2.541\times10^{-20}$ & 200 & 10000 \\
QPP lambda=0.5 vs MC (S=5) & mean & two-sided & \asymci{0.04568}{0.08355}{0.06672} & 0.2519 & 1 & 200 & 10000 \\
QPP lambda=0.5 vs MC (S=5) & CVaR & two-sided & \asymci{-1.009}{0.695}{0.5701} & 0.0002 & 0.002 & 200 & 10000 \\
QPP lambda=0.5 vs MC (S=5) & V1 headroom loss & two-sided & \asymci{-0.35}{0.065}{0.065} & $1.694\times10^{-21}$ & $2.541\times10^{-20}$ & 200 & 10000 \\
QPP lambda=1 vs MC (S=5) & mean & two-sided & \asymci{0.0527}{0.0821}{0.0646} & 0.173 & 0.8649 & 200 & 10000 \\
QPP lambda=1 vs MC (S=5) & CVaR & two-sided & \asymci{-0.9781}{0.6799}{0.5404} & 0.0002 & 0.002 & 200 & 10000 \\
QPP lambda=1 vs MC (S=5) & V1 headroom loss & two-sided & \asymci{-0.345}{0.065}{0.065} & $3.388\times10^{-21}$ & $4.405\times10^{-20}$ & 200 & 10000 \\
mean vs MC (S=5) & mean & two-sided & \asymci{-0.04491}{0.01127}{0.01014} & $9.999\times10^{-5}$ & 0.0012 & 200 & 10000 \\
mean vs MC (S=5) & CVaR & two-sided & \asymci{-0.07551}{0.08169}{0.083144} & 0.08159 & 0.5711 & 200 & 10000 \\
mean vs MC (S=5) & V1 headroom loss & two-sided & \asymci{0.01}{0.01}{0.015} & 0.5 & 1 & 200 & 10000 \\
moment match vs MC (S=5) & mean & two-sided & \asymci{-0.04421}{0.01126}{0.01014} & $9.999\times10^{-5}$ & 0.0012 & 200 & 10000 \\
moment match vs MC (S=5) & CVaR & two-sided & \asymci{-0.0708}{0.0835}{0.08301} & 0.1088 & 0.6527 & 200 & 10000 \\
moment match vs MC (S=5) & V1 headroom loss & two-sided & \asymci{0.01}{0.01}{0.015} & 0.5 & 1 & 200 & 10000 \\
\bottomrule
\end{tabular}
 \caption{Microgrid H1 versus $S = 5$ Monte Carlo, evaluation site $255$. One
Holm family per site.}
\label{tab:microgrid-vs-mc-255}
\end{sidewaystable}

\begin{sidewaystable}[p]
\centering
\footnotesize
\setlength{\tabcolsep}{3pt}
\begin{tabular}{@{}>{\raggedright\arraybackslash}p{0.30\linewidth} >{\raggedright\arraybackslash}p{0.10\linewidth} >{\raggedright\arraybackslash}p{0.07\linewidth} >{\centering\arraybackslash}p{0.15\linewidth} >{\raggedleft\arraybackslash}p{0.10\linewidth} >{\raggedleft\arraybackslash}p{0.10\linewidth} r r@{}}
\toprule
Comparison & Metric & Alt. & Estimate with 95\% CI & $p$ & $p_\mathrm{adj}$ & $n$ & $n_\mathrm{boot}$ \\
\midrule
QPP lambda=0 vs MC (S=5) & mean & two-sided & \asymci{-0.1514}{0.1113}{0.09372} & 0.002 & 0.01367 & 200 & 10000 \\
QPP lambda=0 vs MC (S=5) & CVaR & two-sided & \asymci{-1.701}{0.774}{0.661} & 0.0002 & 0.0024 & 200 & 10000 \\
QPP lambda=0 vs MC (S=5) & V1 headroom loss & two-sided & \asymci{-0.42}{0.07}{0.065} & $1.034\times10^{-25}$ & $1.551\times10^{-24}$ & 200 & 10000 \\
QPP lambda=0.5 vs MC (S=5) & mean & two-sided & \asymci{-0.1376}{0.1139}{0.09407} & 0.006799 & 0.0272 & 200 & 10000 \\
QPP lambda=0.5 vs MC (S=5) & CVaR & two-sided & \asymci{-1.656}{0.789}{0.6576} & 0.0002 & 0.0024 & 200 & 10000 \\
QPP lambda=0.5 vs MC (S=5) & V1 headroom loss & two-sided & \asymci{-0.415}{0.070}{0.0651} & $2.068\times10^{-25}$ & $2.895\times10^{-24}$ & 200 & 10000 \\
QPP lambda=1 vs MC (S=5) & mean & two-sided & \asymci{-0.1317}{0.1136}{0.09273} & 0.008999 & 0.0272 & 200 & 10000 \\
QPP lambda=1 vs MC (S=5) & CVaR & two-sided & \asymci{-1.627}{0.807}{0.6455} & 0.0002 & 0.0024 & 200 & 10000 \\
QPP lambda=1 vs MC (S=5) & V1 headroom loss & two-sided & \asymci{-0.415}{0.070}{0.070} & $2.068\times10^{-25}$ & $2.895\times10^{-24}$ & 200 & 10000 \\
mean vs MC (S=5) & mean & two-sided & \asymci{-0.04156}{0.02528}{0.03058} & 0.0011 & 0.009899 & 200 & 10000 \\
mean vs MC (S=5) & CVaR & two-sided & \asymci{0.1726}{0.23032}{0.2943} & 0.1696 & 0.3384 & 200 & 10000 \\
mean vs MC (S=5) & V1 headroom loss & two-sided & \asymci{0.05}{0.025}{0.03} & 0.001953 & 0.01367 & 200 & 10000 \\
moment match vs MC (S=5) & mean & two-sided & \asymci{-0.04118}{0.02481}{0.031309} & 0.0016 & 0.0128 & 200 & 10000 \\
moment match vs MC (S=5) & CVaR & two-sided & \asymci{0.1715}{0.22833}{0.2805} & 0.1692 & 0.3384 & 200 & 10000 \\
moment match vs MC (S=5) & V1 headroom loss & two-sided & \asymci{0.05}{0.025}{0.03} & 0.001953 & 0.01367 & 200 & 10000 \\
\bottomrule
\end{tabular}
 \caption{Microgrid H1 versus $S = 5$ Monte Carlo, evaluation site $279$. One
Holm family per site.}
\label{tab:microgrid-vs-mc-279}
\end{sidewaystable}

\FloatBarrier
\subsection{Exploratory objective and coupling sensitivity}
\label{sec:appendix-review-sensitivity}

\paragraph{Primary and exploratory analyses.}
The primary endpoints, evaluation sites, and decision thresholds were
specified before conducting the sensitivity analyses reported in this
section. The exact-stream evaluation, cross-site multiplicity analysis, and
objective and coupling sensitivity studies are therefore labeled
exploratory. These analyses do not alter the definitions or conclusions of
the primary evaluations.

\subsubsection{Objective sensitivity}

The objective audit uses an independent seed block and $n=1{,}000$ CRN
repeats per site, so each empirical $\CVaR_{0.1}$ averages 100 tail outcomes.
All arms retain $K=5$, $P=64$, $H=24$, ten CEM iterations and ten elites.
The objective audit changes one item at a time from QPP at $\alpha=0.2$:
CVaR after summing the full trajectory, planning at $\alpha=0.1$, or an
opt-in terminal penalty equal to the running squared-overload scale applied
to terminal headroom shortfall. These variants are evaluated separately from
the primary objective.

\begin{sidewaystable}[p]
\centering
\footnotesize
\setlength{\tabcolsep}{3pt}
\begin{tabular}{@{}l >{\raggedright\arraybackslash}p{0.28\linewidth} *{3}{>{\centering\arraybackslash}p{0.18\linewidth}}@{}}
\toprule
Site & Variant & $\Delta$ mean with 95\% CI & $\Delta$ CVaR$_{0.1}$ with 95\% CI & $\Delta$ V1 pp with 95\% CI \\
\midrule
242 & trajectory CVaR, $\alpha=0.2$ & \asymci{0.000}{0.000}{0.000} & \asymci{0.000}{0.000}{0.000} & \asymci{0.000}{0.000}{0.000} \\
242 & per-step CVaR, $\alpha=0.1$ & \asymci{0.005}{0.002}{0.002} & \asymci{0.016}{0.015}{0.017} & \asymci{-0.100}{0.500}{0.500} \\
242 & terminal-aligned penalty & \asymci{0.001}{0.002}{0.001} & \asymci{0.001}{0.013}{0.015} & \asymci{0.300}{0.500}{0.500} \\
253 & trajectory CVaR, $\alpha=0.2$ & \asymci{0.000}{0.000}{0.000} & \asymci{0.000}{0.000}{0.000} & \asymci{0.000}{0.000}{0.000} \\
253 & per-step CVaR, $\alpha=0.1$ & \asymci{0.004}{0.001}{0.002} & \asymci{0.019}{0.013}{0.015} & \asymci{0.100}{0.300}{0.400} \\
253 & terminal-aligned penalty & \asymci{0.001}{0.002}{0.001} & \asymci{0.010}{0.014}{0.014} & \asymci{0.100}{0.300}{0.302} \\
255 & trajectory CVaR, $\alpha=0.2$ & \asymci{0.000}{0.000}{0.000} & \asymci{0.000}{0.000}{0.000} & \asymci{0.000}{0.000}{0.000} \\
255 & per-step CVaR, $\alpha=0.1$ & \asymci{0.002}{0.002}{0.001} & \asymci{0.005}{0.012}{0.012} & \asymci{0.000}{0.300}{0.300} \\
255 & terminal-aligned penalty & \asymci{-0.000}{0.002}{0.001} & \asymci{0.003}{0.014}{0.014} & \asymci{0.100}{0.300}{0.400} \\
279 & trajectory CVaR, $\alpha=0.2$ & \asymci{0.000}{0.000}{0.000} & \asymci{0.000}{0.000}{0.000} & \asymci{0.000}{0.000}{0.000} \\
279 & per-step CVaR, $\alpha=0.1$ & \asymci{0.002}{0.002}{0.002} & \asymci{0.001}{0.022}{0.021} & \asymci{0.200}{0.200}{0.300} \\
279 & terminal-aligned penalty & \asymci{-0.001}{0.002}{0.003} & \asymci{-0.003}{0.020}{0.019} & \asymci{0.000}{0.400}{0.400} \\
\bottomrule
\end{tabular}
 \caption{Objective sensitivity relative to the original per-step,
$\alpha=0.2$ QPP arm. Entries are variant-minus-original paired differences
with descriptive percentile-bootstrap intervals; lower cost is better, and
V1 differences are percentage points.
Total-trajectory and per-step scores select the same plans here because the
shared-rank cost paths remain ordered consistently across time.}
\label{tab:objective-sensitivity}
\end{sidewaystable}

\subsubsection{Coupling and node sensitivity}

The separate $n=200$ factorial arms use equal-weight stratified scenarios and the same
total-trajectory objective. Fixed versus randomized nodes isolates node
randomization; shared versus time-varying rank assignments isolates temporal
coupling. It is a mechanism diagnostic, not a claim that any arm is a better
approximation of the independent execution process.

\begin{sidewaystable}[p]
\centering
\footnotesize
\setlength{\tabcolsep}{3pt}
\begin{tabular}{@{}l >{\raggedright\arraybackslash}p{0.28\linewidth} *{3}{>{\centering\arraybackslash}p{0.18\linewidth}}@{}}
\toprule
Site & Variant & $\Delta$ mean with 95\% CI & $\Delta$ CVaR$_{0.1}$ with 95\% CI & $\Delta$ V1 pp with 95\% CI \\
\midrule
242 & fixed, shared ranks & \asymci{-0.016}{0.007}{0.007} & \asymci{-0.042}{0.061}{0.051} & \asymci{0.000}{2.000}{2.000} \\
242 & fixed, time-varying ranks & \asymci{-0.009}{0.061}{0.069} & \asymci{1.061}{0.411}{0.439} & \asymci{63.000}{7.000}{6.500} \\
242 & randomized, shared ranks & \asymci{0.085}{0.055}{0.061} & \asymci{0.938}{0.370}{0.369} & \asymci{39.000}{6.500}{7.000} \\
242 & randomized, time-varying ranks & \asymci{0.019}{0.063}{0.069} & \asymci{1.053}{0.441}{0.458} & \asymci{61.000}{7.000}{7.000} \\
253 & fixed, shared ranks & \asymci{-0.014}{0.004}{0.004} & \asymci{-0.041}{0.033}{0.031} & \asymci{0.000}{0.000}{0.000} \\
253 & fixed, time-varying ranks & \asymci{-0.084}{0.045}{0.051} & \asymci{0.740}{0.330}{0.319} & \asymci{56.500}{7.000}{6.500} \\
253 & randomized, shared ranks & \asymci{0.019}{0.041}{0.045} & \asymci{0.717}{0.292}{0.255} & \asymci{38.000}{7.000}{7.000} \\
253 & randomized, time-varying ranks & \asymci{-0.057}{0.047}{0.050} & \asymci{0.744}{0.333}{0.323} & \asymci{54.500}{7.000}{7.000} \\
255 & fixed, shared ranks & \asymci{-0.015}{0.006}{0.006} & \asymci{-0.012}{0.047}{0.065} & \asymci{-0.500}{1.000}{0.500} \\
255 & fixed, time-varying ranks & \asymci{-0.145}{0.034}{0.040} & \asymci{0.471}{0.247}{0.262} & \asymci{40.000}{6.500}{6.500} \\
255 & randomized, shared ranks & \asymci{-0.036}{0.035}{0.038} & \asymci{0.535}{0.235}{0.238} & \asymci{28.500}{6.000}{6.500} \\
255 & randomized, time-varying ranks & \asymci{-0.123}{0.035}{0.039} & \asymci{0.479}{0.245}{0.261} & \asymci{39.000}{6.500}{6.500} \\
279 & fixed, shared ranks & \asymci{-0.014}{0.008}{0.007} & \asymci{-0.022}{0.059}{0.057} & \asymci{0.500}{0.500}{1.000} \\
279 & fixed, time-varying ranks & \asymci{0.022}{0.067}{0.076} & \asymci{1.262}{0.427}{0.411} & \asymci{55.500}{7.000}{6.500} \\
279 & randomized, shared ranks & \asymci{0.130}{0.062}{0.068} & \asymci{1.150}{0.370}{0.365} & \asymci{35.500}{6.500}{6.500} \\
279 & randomized, time-varying ranks & \asymci{0.042}{0.063}{0.067} & \asymci{1.080}{0.397}{0.396} & \asymci{51.000}{7.000}{7.000} \\
\bottomrule
\end{tabular}
 \caption{Matched $2\times2$ coupling diagnostic relative to the original QPP
arm. Time-varying ranks materially worsen the realized tail and V1 at every
site; randomized shared ranks are also worse, while fixed shared ranks stay
close. This separates temporal coupling from determinism without changing the
trajectory budget.}
\label{tab:coupling-factorial}
\end{sidewaystable}

\FloatBarrier
\clearpage
\subsection{Microgrid: trajectory-count sweep, IID and coupled}
\label{sec:appendix-microgrid-crossover}

Figure~\ref{fig:microgrid-crossover} gives the cross-site pattern before the
summary table. With the search allocation fixed, all four IID comparisons
meet the trajectory-count rule because their one-sided lower bounds remain
above $5\%$ at $S=200$. The clipped coupled sampler does not meet its
equivalence rule at any site.

\begin{figure}[h]
\centering
\includegraphics[width=0.98\textwidth]{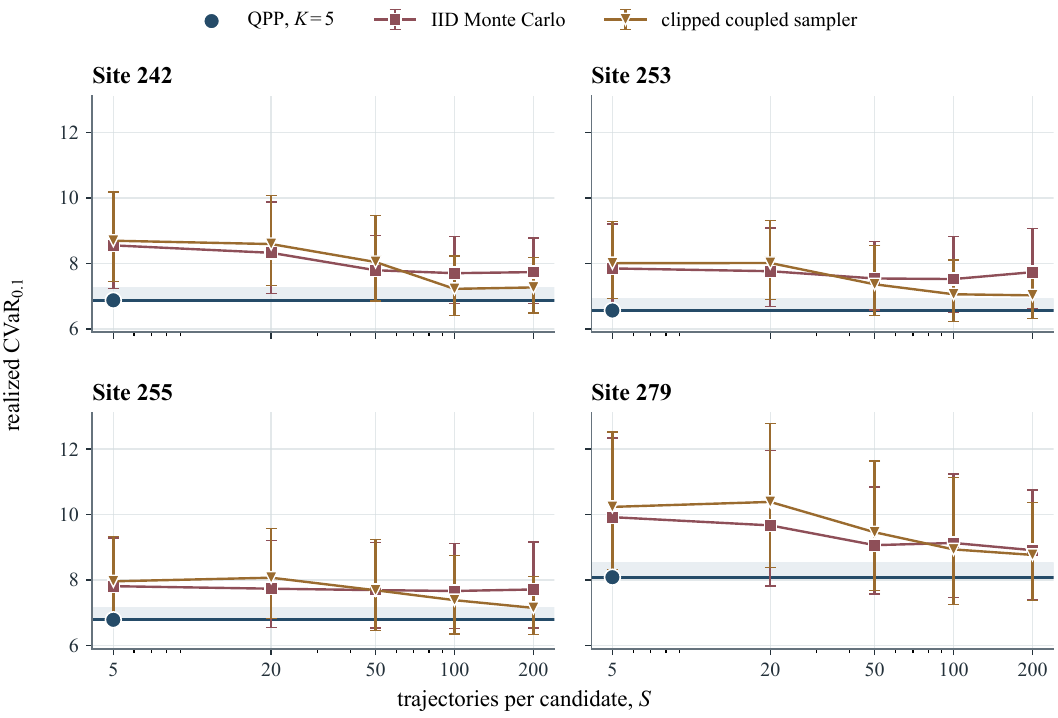}
\caption{Microgrid trajectory-count sweeps at all four held-out sites. Each
panel compares IID Monte Carlo and the clipped coupled sampler with $K=5$
QPP as $S$ increases. Error bars are bootstrap $95\%$ intervals; the narrow
blue region is the $5\%$ sampler-denominated comparison band above QPP. The
CEM population, elite count and iteration count are fixed, so this is not a
best-allocation total-budget frontier. $n=200$ per point.}
\label{fig:microgrid-crossover}
\end{figure}

\begin{table}[t]
\centering
\small
\setlength{\tabcolsep}{4.5pt}
\begin{tabular}{@{}lcccc@{}}
\toprule
Site & IID gap with 95\% lower bound & B3 & Clipped gap with 95\% CI & Equivalence \\
\midrule
242 & \asymci{11.1\%}{4.7\%}{\infty} & Satisfied & \asymci{5.4\%}{4.3\%}{4.7\%} & Not satisfied \\
253 & \asymci{15.1\%}{5.4\%}{\infty} & Satisfied & \asymci{6.6\%}{3.3\%}{3.4\%} & Not satisfied \\
255 & \asymci{12.0\%}{5.0\%}{\infty} & Satisfied & \asymci{5.1\%}{3.5\%}{3.5\%} & Not satisfied \\
279 & \asymci{9.3\%}{4.1\%}{\infty} & Satisfied & \asymci{7.8\%}{4.5\%}{4.9\%} & Not satisfied \\
\bottomrule
\end{tabular}
 \caption{Cross-site decision summary at $S=200$. Positive gaps mean that the
sampler has higher realized $\CVaR_{0.1}$ than QPP. IID entries pair the point
estimate with the one-sided $95\%$ lower bound used by B3; clipped-coupled
entries use the two-sided interval required by the equivalence diagnostic.
All four B3 criteria are satisfied by the lower-bound branch, while none of
the four equivalence criteria is satisfied.}
\label{tab:crossover-summary}
\end{table}

\FloatBarrier
\subsection{Microgrid: quantile ensemble versus budget-matched PETS}
\label{sec:appendix-microgrid-ensemble}
\label{sec:appendix-ensemble-detail}

\subsubsection{Design and cross-site summary}

Replay windows come from each site's final $30\%$ of rows. A valid start is
an hour-of-day-zero row with at least $H+1=25$ rows remaining, yielding $328$
eligible starts; each repeat draws one start uniformly with replacement and
uses the same draw for every arm. The cited $200$ repeats contain $152$
distinct starts. Source admissibility is satisfied (pooled PIT deviation
$0.0113$).
The design check that selects replay is also conditional on the eventual four
sites (median reserve gap $0.1750$ against $0.15$; site gaps
$0.025$/$0.225$/$0.200$/$0.150$ at $242$/$253$/$255$/$279$). The quantile
ensemble beats PETS on tail risk at three sites and on $V1$ at all four, but
the single-model $25$-level arm matches those margins, PETS is
indistinguishable from plain Monte Carlo, and the epistemic-premium gates are
unsupported. Accordingly the result supports deterministic quantile
evaluation on replay, not ensemble allocation or epistemic-risk claims.

\subsubsection{Site-level results}

\begin{sidewaystable}[p]
\centering
\footnotesize
\setlength{\tabcolsep}{3pt}
\begin{tabular}{@{}>{\raggedright\arraybackslash}p{0.30\linewidth} >{\raggedright\arraybackslash}p{0.10\linewidth} >{\raggedright\arraybackslash}p{0.07\linewidth} >{\centering\arraybackslash}p{0.15\linewidth} >{\raggedleft\arraybackslash}p{0.10\linewidth} >{\raggedleft\arraybackslash}p{0.10\linewidth} r r@{}}
\toprule
Comparison & Metric & Alt. & Estimate with 95\% CI & $p$ & $p_\mathrm{adj}$ & $n$ & $n_\mathrm{boot}$ \\
\midrule
trajectory ensemble (M=5, K=5) vs PETS (M=5, S=5) & mean & two-sided & \asymci{0.1748}{0.0089}{0.0085} & $9.999\times10^{-5}$ & 0.0009999 & 200 & 10000 \\
trajectory ensemble (M=5, K=5) vs PETS (M=5, S=5) & CVaR & lower & \asymci{0.1435}{\infty}{0.0213} & 1 & 1 & 200 & 10000 \\
trajectory ensemble (M=5, K=5) vs PETS (M=5, S=5) & violation & lower & \asymci{-0.11}{\infty}{0.035} & $2.384\times10^{-7}$ & $2.861\times10^{-6}$ & 200 & 10000 \\
QPP ensemble (M=1, K=25) vs PETS (M=5, S=5) & mean & two-sided & \asymci{0.1692}{0.0088}{0.0088} & $9.999\times10^{-5}$ & 0.0009999 & 200 & 10000 \\
QPP ensemble (M=1, K=25) vs PETS (M=5, S=5) & CVaR & two-sided & \asymci{0.1394}{0.0243}{0.0262} & 0.0002 & 0.0012 & 200 & 10000 \\
QPP ensemble (M=1, K=25) vs PETS (M=5, S=5) & violation & two-sided & \asymci{-0.115}{0.045}{0.040} & $2.384\times10^{-7}$ & $2.861\times10^{-6}$ & 200 & 10000 \\
trajectory ensemble (M=25, K=1) vs PETS (M=5, S=5) & mean & two-sided & \asymci{-0.07967}{0.00827}{0.00805} & $9.999\times10^{-5}$ & 0.0009999 & 200 & 10000 \\
trajectory ensemble (M=25, K=1) vs PETS (M=5, S=5) & CVaR & lower & \asymci{-0.07639}{\infty}{0.01803} & $9.999\times10^{-5}$ & 0.0009999 & 200 & 10000 \\
trajectory ensemble (M=25, K=1) vs PETS (M=5, S=5) & violation & lower & \asymci{0.055}{\infty}{0.030} & 1 & 1 & 200 & 10000 \\
MC (S=25) vs PETS (M=5, S=5) & mean & two-sided & \asymci{-0.0009754}{0.0085136}{0.0085204} & 0.8297 & 1 & 200 & 10000 \\
MC (S=25) vs PETS (M=5, S=5) & CVaR & two-sided & \asymci{0.00427}{0.02637}{0.02485} & 0.8115 & 1 & 200 & 10000 \\
MC (S=25) vs PETS (M=5, S=5) & violation & two-sided & \asymci{0.02}{0.02}{0.025} & 0.2188 & 1 & 200 & 10000 \\
\bottomrule
\end{tabular}
 \caption{Microgrid ensemble comparison at evaluation site $242$ (sites
$253$, $255$, $279$ follow): five models with five quantile levels each,
one reference model with $25$ levels, a $25$-model mean-path ensemble and
$25$-sample plain Monte Carlo, each against a budget-matched PETS-style
ensemble. Every arm uses $25$ evaluations per candidate and step and is
executed by held-out trace replay. One Holm family per site.}
\label{tab:ensemble-vs-pets-242}
\end{sidewaystable}

\begin{sidewaystable}[p]
\centering
\footnotesize
\setlength{\tabcolsep}{3pt}
\begin{tabular}{@{}>{\raggedright\arraybackslash}p{0.30\linewidth} >{\raggedright\arraybackslash}p{0.10\linewidth} >{\raggedright\arraybackslash}p{0.07\linewidth} >{\centering\arraybackslash}p{0.15\linewidth} >{\raggedleft\arraybackslash}p{0.10\linewidth} >{\raggedleft\arraybackslash}p{0.10\linewidth} r r@{}}
\toprule
Comparison & Metric & Alt. & Estimate with 95\% CI & $p$ & $p_\mathrm{adj}$ & $n$ & $n_\mathrm{boot}$ \\
\midrule
trajectory ensemble (M=5, K=5) vs PETS (M=5, S=5) & mean & two-sided & \asymci{0.06309}{0.07691}{0.06221} & 0.07749 & 0.5424 & 200 & 10000 \\
trajectory ensemble (M=5, K=5) vs PETS (M=5, S=5) & CVaR & lower & \asymci{-1.01}{\infty}{0.5121} & $9.999\times10^{-5}$ & 0.0009999 & 200 & 10000 \\
trajectory ensemble (M=5, K=5) vs PETS (M=5, S=5) & violation & lower & \asymci{-0.28}{\infty}{0.05} & $1.388\times10^{-17}$ & $1.665\times10^{-16}$ & 200 & 10000 \\
QPP ensemble (M=1, K=25) vs PETS (M=5, S=5) & mean & two-sided & \asymci{0.06054}{0.07420}{0.06216} & 0.08509 & 0.5424 & 200 & 10000 \\
QPP ensemble (M=1, K=25) vs PETS (M=5, S=5) & CVaR & two-sided & \asymci{-0.9997}{0.6643}{0.5935} & 0.0002 & 0.0016 & 200 & 10000 \\
QPP ensemble (M=1, K=25) vs PETS (M=5, S=5) & violation & two-sided & \asymci{-0.275}{0.060}{0.060} & $5.551\times10^{-17}$ & $6.106\times10^{-16}$ & 200 & 10000 \\
trajectory ensemble (M=25, K=1) vs PETS (M=5, S=5) & mean & two-sided & \asymci{-0.0529}{0.00614}{0.00641} & $9.999\times10^{-5}$ & 0.0009999 & 200 & 10000 \\
trajectory ensemble (M=25, K=1) vs PETS (M=5, S=5) & CVaR & lower & \asymci{-0.0139}{\infty}{0.03271} & 0.2165 & 1 & 200 & 10000 \\
trajectory ensemble (M=25, K=1) vs PETS (M=5, S=5) & violation & lower & \asymci{0.05}{\infty}{0.025} & 1 & 1 & 200 & 10000 \\
MC (S=25) vs PETS (M=5, S=5) & mean & two-sided & \asymci{-0.01139}{0.01961}{0.013598} & 0.2013 & 1 & 200 & 10000 \\
MC (S=25) vs PETS (M=5, S=5) & CVaR & two-sided & \asymci{-0.09467}{0.18913}{0.11452} & 0.2158 & 1 & 200 & 10000 \\
MC (S=25) vs PETS (M=5, S=5) & violation & two-sided & \asymci{0.02}{0.03}{0.03} & 0.3438 & 1 & 200 & 10000 \\
\bottomrule
\end{tabular}
 \caption{Microgrid, quantile ensemble versus budget-matched PETS,
evaluation site $253$. One Holm family per site.}
\label{tab:ensemble-vs-pets-253}
\end{sidewaystable}

\begin{sidewaystable}[p]
\centering
\footnotesize
\setlength{\tabcolsep}{3pt}
\begin{tabular}{@{}>{\raggedright\arraybackslash}p{0.30\linewidth} >{\raggedright\arraybackslash}p{0.10\linewidth} >{\raggedright\arraybackslash}p{0.07\linewidth} >{\centering\arraybackslash}p{0.15\linewidth} >{\raggedleft\arraybackslash}p{0.10\linewidth} >{\raggedleft\arraybackslash}p{0.10\linewidth} r r@{}}
\toprule
Comparison & Metric & Alt. & Estimate with 95\% CI & $p$ & $p_\mathrm{adj}$ & $n$ & $n_\mathrm{boot}$ \\
\midrule
trajectory ensemble (M=5, K=5) vs PETS (M=5, S=5) & mean & two-sided & \asymci{0.102}{0.05432}{0.0449} & $9.999\times10^{-5}$ & 0.0009999 & 200 & 10000 \\
trajectory ensemble (M=5, K=5) vs PETS (M=5, S=5) & CVaR & lower & \asymci{-0.7446}{\infty}{0.3082} & $9.999\times10^{-5}$ & 0.0009999 & 200 & 10000 \\
trajectory ensemble (M=5, K=5) vs PETS (M=5, S=5) & violation & lower & \asymci{-0.185}{\infty}{0.045} & $7.276\times10^{-12}$ & $8.731\times10^{-11}$ & 200 & 10000 \\
QPP ensemble (M=1, K=25) vs PETS (M=5, S=5) & mean & two-sided & \asymci{0.09645}{0.05510}{0.04645} & $9.999\times10^{-5}$ & 0.0009999 & 200 & 10000 \\
QPP ensemble (M=1, K=25) vs PETS (M=5, S=5) & CVaR & two-sided & \asymci{-0.7609}{0.4171}{0.3645} & 0.0002 & 0.0012 & 200 & 10000 \\
QPP ensemble (M=1, K=25) vs PETS (M=5, S=5) & violation & two-sided & \asymci{-0.185}{0.055}{0.050} & $1.455\times10^{-11}$ & $1.601\times10^{-10}$ & 200 & 10000 \\
trajectory ensemble (M=25, K=1) vs PETS (M=5, S=5) & mean & two-sided & \asymci{-0.03784}{0.00449}{0.00436} & $9.999\times10^{-5}$ & 0.0009999 & 200 & 10000 \\
trajectory ensemble (M=25, K=1) vs PETS (M=5, S=5) & CVaR & lower & \asymci{-0.03502}{\infty}{0.01970} & 0.0033 & 0.0165 & 200 & 10000 \\
trajectory ensemble (M=25, K=1) vs PETS (M=5, S=5) & violation & lower & \asymci{0.01}{\infty}{0.015} & 1 & 1 & 200 & 10000 \\
MC (S=25) vs PETS (M=5, S=5) & mean & two-sided & \asymci{-0.004585}{0.010275}{0.007781} & 0.3963 & 1 & 200 & 10000 \\
MC (S=25) vs PETS (M=5, S=5) & CVaR & two-sided & \asymci{-0.04817}{0.09493}{0.05963} & 0.1762 & 0.7047 & 200 & 10000 \\
MC (S=25) vs PETS (M=5, S=5) & violation & two-sided & \asymci{0}{0.02}{0.02} & 1 & 1 & 200 & 10000 \\
\bottomrule
\end{tabular}
 \caption{Microgrid, quantile ensemble versus budget-matched PETS,
evaluation site $255$. One Holm family per site.}
\label{tab:ensemble-vs-pets-255}
\end{sidewaystable}

\begin{sidewaystable}[p]
\centering
\footnotesize
\setlength{\tabcolsep}{3pt}
\begin{tabular}{@{}>{\raggedright\arraybackslash}p{0.30\linewidth} >{\raggedright\arraybackslash}p{0.10\linewidth} >{\raggedright\arraybackslash}p{0.07\linewidth} >{\centering\arraybackslash}p{0.15\linewidth} >{\raggedleft\arraybackslash}p{0.10\linewidth} >{\raggedleft\arraybackslash}p{0.10\linewidth} r r@{}}
\toprule
Comparison & Metric & Alt. & Estimate with 95\% CI & $p$ & $p_\mathrm{adj}$ & $n$ & $n_\mathrm{boot}$ \\
\midrule
trajectory ensemble (M=5, K=5) vs PETS (M=5, S=5) & mean & two-sided & \asymci{0.09885}{0.03230}{0.02645} & $9.999\times10^{-5}$ & 0.0009999 & 200 & 10000 \\
trajectory ensemble (M=5, K=5) vs PETS (M=5, S=5) & CVaR & lower & \asymci{-0.3026}{\infty}{0.1805} & 0.0004 & 0.0028 & 200 & 10000 \\
trajectory ensemble (M=5, K=5) vs PETS (M=5, S=5) & violation & lower & \asymci{-0.23}{\infty}{0.05} & $1.421\times10^{-14}$ & $1.705\times10^{-13}$ & 200 & 10000 \\
QPP ensemble (M=1, K=25) vs PETS (M=5, S=5) & mean & two-sided & \asymci{0.09529}{0.03061}{0.02631} & $9.999\times10^{-5}$ & 0.0009999 & 200 & 10000 \\
QPP ensemble (M=1, K=25) vs PETS (M=5, S=5) & CVaR & two-sided & \asymci{-0.3073}{0.2580}{0.2018} & 0.0007999 & 0.0048 & 200 & 10000 \\
QPP ensemble (M=1, K=25) vs PETS (M=5, S=5) & violation & two-sided & \asymci{-0.23}{0.06}{0.055} & $2.842\times10^{-14}$ & $3.126\times10^{-13}$ & 200 & 10000 \\
trajectory ensemble (M=25, K=1) vs PETS (M=5, S=5) & mean & two-sided & \asymci{-0.08109}{0.00932}{0.00974} & $9.999\times10^{-5}$ & 0.0009999 & 200 & 10000 \\
trajectory ensemble (M=25, K=1) vs PETS (M=5, S=5) & CVaR & lower & \asymci{-0.02996}{\infty}{0.04587} & 0.1375 & 0.6874 & 200 & 10000 \\
trajectory ensemble (M=25, K=1) vs PETS (M=5, S=5) & violation & lower & \asymci{0.03}{\infty}{0.02} & 1 & 1 & 200 & 10000 \\
MC (S=25) vs PETS (M=5, S=5) & mean & two-sided & \asymci{-0.002265}{0.011775}{0.011303} & 0.6988 & 1 & 200 & 10000 \\
MC (S=25) vs PETS (M=5, S=5) & CVaR & two-sided & \asymci{-0.0009552}{0.0938948}{0.0681352} & 0.9813 & 1 & 200 & 10000 \\
MC (S=25) vs PETS (M=5, S=5) & violation & two-sided & \asymci{0.01}{0.025}{0.03} & 0.7266 & 1 & 200 & 10000 \\
\bottomrule
\end{tabular}
 \caption{Microgrid, quantile ensemble versus budget-matched PETS,
evaluation site $279$. One Holm family per site.}
\label{tab:ensemble-vs-pets-279}
\label{jmlrend}
\end{sidewaystable}
 
\end{document}